\documentclass{article}
\usepackage{arxiv}
\usepackage[utf8]{inputenc} 
\usepackage[T1]{fontenc}    
\usepackage{hyperref}       
\usepackage{url}            
\usepackage{booktabs}       
\usepackage{amsthm}
\usepackage{amsfonts}       
\usepackage{nicefrac}       
\usepackage{microtype}      
\usepackage{lipsum}
\usepackage{amsmath}
\usepackage{graphicx}
\usepackage{natbib}
\usepackage{doi}
\usepackage{amssymb}
\usepackage{subcaption}  
\newtheorem{theorem}{Theorem}
\newtheorem{lemma}{Lemma}

\usepackage{xcolor}
\usepackage{bm}
\usepackage{algpseudocode}  
\usepackage{booktabs} 
\usepackage{makecell}
\usepackage{multirow}

\usepackage{amsthm}
\newtheorem{definition}{Definition}

\usepackage{amsmath, amssymb}
\usepackage{algorithm}
\usepackage{geometry}
\title{BVFLMSP : Bayesian Vertical Federated Learning for Multimodal Survival with Privacy}

\renewcommand{\shorttitle}{\textit{arXiv} Template}

\hypersetup{
pdftitle={A template for the arxiv style},
pdfsubject={q-bio.NC, q-bio.QM},
pdfauthor={David S.~Hippocampus, Elias D.~Striatum},
pdfkeywords={First keyword, Second keyword, More},
}

\begin{document}

\author{
Abhilash Kar$^{1}$\thanks{First Author and Second Author contributed equally to this work.}
\quad Basisth Saha$^{2}$\footnotemark[1]
\quad Tanmay Sen$^{1}$\thanks{Corresponding author's email sentanmay518@gmail.com}
\quad Biswabrata Pradhan$^{1}$
}


\maketitle

\begin{center}
$^{1}$SQC and OR Unit, Indian Statistical Institute, Kolkata, India\\
$^{2}$Department of Statistics, Presidency University, Kolkata, India\\

\end{center}

\begin{abstract}

    Multimodal time-to-event prediction often requires integrating sensitive data distributed across multiple parties, making centralized model training impractical due to privacy constraints. At the same time, most existing multimodal survival models produce single deterministic predictions without indicating how confident the model is in its estimates, which can limit their reliability in real-world decision making. To address these challenges, we propose BVFLMSP, a Bayesian Vertical Federated Learning (VFL) framework for multimodal time-to-event analysis based on a Split Neural Network architecture. In BVFLMSP, each client independently models a specific data modality using a Bayesian neural network, while a central server aggregates intermediate representations to perform survival risk prediction. To enhance privacy, we integrate differential privacy mechanisms by perturbing client side representations before transmission, providing formal privacy guarantees against information leakage during federated training.
    We first evaluate our Bayesian multimodal survival model against widely used single modality survival baselines and the centralized multimodal baseline MultiSurv. Across multimodal settings, the proposed method shows consistent improvements in discrimination performance, with up to 0.02 higher C-index compared to MultiSurv. We then compare federated and centralized learning under varying privacy budgets across different modality combinations, highlighting the tradeoff between predictive performance and privacy. Experimental results show that BVFLMSP effectively includes multimodal data, improves survival prediction over existing baselines, and remains robust under strict privacy constraints while providing uncertainty estimates.

\end{abstract}


\section{Introduction}

Uncertainty is present in almost every real world decision making process. In predictive modeling, especially in medical and industrial applications, it is important not only to make accurate predictions but also to understand how confident a model is in those predictions. This is particularly critical in high stakes settings such as medical prognosis, where decisions can directly affect patient outcomes, and in industrial domains, where equipment reliability impacts the production process. 
Diseases such as cancer account for nearly 10 million deaths worldwide each year, making accurate time to event prediction essential for treatment planning and risk assessment.
Similarly, sudden failures of industrial equipment can significantly disrupt production processes, therefore reliability assessment and early maintenance strategies are necessary to prevent such breakdowns.
For that, survival analysis is important to predict the time until critical events occur, enabling decisions regarding patient treatment strategies and the scheduling of industrial equipment maintenance. 
The rapid adoption of machine learning and deep learning has led to major advances in predictive modeling. However, early machine learning and deep learning models produce point estimation and do not explicitly model uncertainty, often resulting in wrong overconfident point estimations. In predictive analysis, uncertainty is commonly categorized into two types: aleatoric uncertainty, which arises from inherent noise and randomness in the data, and epistemic uncertainty, which arises from limited knowledge about the model. While aleatoric uncertainty cannot be reduced, epistemic uncertainty can be mitigated as more data and better models become available. Bayesian neural networks provide a principled framework to model both types of uncertainty and produce uncertainty aware predictions.


In both the medical and industrial domains, the focus extends beyond merely identifying or classifying types of failures. A critical concern lies in assessing the corrective decisions regarding an industrial or medical process, analyzing the risk factors influencing the durability or survivability of subjects under study or understanding the nature of associated characteristics like the recurrence of a disease, etc. To address these problems, we move beyond classification models and use survival analysis.
There are various methods used in survival analysis, including traditional methods such as the Kaplan–Meier (\cite{kaplan1958nonparametric}) and Nelson–Aalen estimators, semi-parametric models such as the Cox proportional hazards model (\cite{cox1972regression}), Accelerated Failure Time (AFT) models, and machine learning approaches such as random survival forests (\cite{ishwaran2008random}) and boosting-based methods (\cite{binder2008allowing}). 


Apart from that, several deep learning based survival models have been developed, including DeepSurv \cite{katzman2018deepsurv}, DeepHit \cite{lee2018deephit}, DeepAFT \cite{norman2024deepaft}, NNet-Survival \cite{gensheimer2019scalable}, and FSL-BDP \cite{amed2026fsl}. These models improve predictive performance by capturing complex non-linear relationships in the data. However, most of these approaches are designed for single-modality inputs and do not naturally support the integration of heterogeneous data sources. 

Modern studies focus on extracting meaningful information from various types of data, such as tabular data, text data, image data, sensor data, signal data, genomic and molecular data, for further research and development \cite{giannakos2019multimodal}. 
Several methodologies have been developed for the analysis of multimodal data, including encoders to process data from different modalities into a common latent space, various fusion methodologies have been developed \cite{pawlowski2023effective} including concatenation based fusion methods\cite{liang2019mcfnet}, attention based fusion methods\cite{hori2017attention}, canonical correlation analysis for multimodal fusion\cite{zhang2021feature}, etc. Effectively including such multimodal data is essential for accurate and reliable survival prediction. To address the limitations of single-modality survival models, prior studies have introduced multimodal approaches such as Multisurv(\cite{vale2020multisurv}), DeepMMSA(\cite{wu2021deepmmsa}), MMsurv(\cite{yang2025mmsurv}). While these models improve predictive performance, high-stakes clinical decision-making requires more than accurate point estimates. It also requires the ability to quantify uncertainty in the predictions, which can help reduce the risk of incorrect decisions.

The increasing use of machine learning and deep learning in medical, industrial, and other domains has highlighted the need for collaborative and privacy preserving modeling approaches. Federated learning addresses this need by enabling multiple clients and a central server to train models jointly without sharing raw data. Depending on how data are distributed across parties, federated learning can be categorized into horizontal and vertical settings.

Although federated learning has been explored for survival analysis, existing studies have mainly focused on horizontal data partitioning. In contrast, vertical federated learning for survival analysis remains largely underexplored, particularly in the presence of uncertainty modeling and explicit privacy protection. Moreover, while federated learning keeps raw data at the client side, it does not inherently defend against adversarial attacks, such as information leakage from intermediate representations. Therefore, additional privacy mechanisms are necessary to mitigate these risks and strengthen privacy guarantees in vertical federated learning. Motivated by these gaps, we propose a privacy preserving and uncertainty aware framework that addresses the following challenges:
    \begin{itemize}
    \item We implement a Split Neural Network(SplitNN) approach based Vertical Federated Learning (VFL) model for survival studies of multimodal data of cancer patients, where each data specific sub model is a client and the server contains labels file, fusion layer, final prediction head, and optimizer.
    \item We propose a Bayesian framework for the entire structure of our VFL model with the motivation to establish a self aware model which can account for the uncertainties in its predictions.
    \item We propose a client-side defense mechanism in SplitNN-based vertical federated learning by adding differential privacy noise to embedding representations before transmission to the server, with the goal of reducing the risk of sensitive information leakage through intermediate feature representations.
    \item We provide a theoretical analysis of the proposed privacy-preserving vertical federated learning framework, including convergence guarantees under differential privacy noise and a formal characterization of the ($\epsilon, \delta$)-privacy budget.
    \item We evaluate the proposed method on multimodal cancer survival datasets and compare it with established single-modality, multimodal, centralized, and federated baselines under varying privacy budgets, demonstrating robust predictive performance and stable convergence in privacy-constrained settings.
\end{itemize}
\section{Related Work}
\paragraph{Survival Analysis Using Machine Learning or Deep Learning:} Predictive analysis has been a prominent area of research and subsequent development through frequent introductions of advancements through methodologies that cover the gaps left behind by the methods prior to them. This realm of studies contains survival analysis or time-to-event analysis, which deals with estimation that includes truncated, censored, unorganized data and thus demands for methods that can afford to deal with these issues and learn the patterns within the data effectively. Early models consisted of non-parametric models like Kaplan-Meier estimator(\cite{kaplan1958nonparametric}), which is relevant even to this day and is used as baseline model for comparison, later on semi-parametric models were developed, like Cox-Proportional Hazards model(\cite{cox1972regression}), the decade of \(1990s\) was the time when the seed of neural network methods in survival analysis was sown (\cite{faraggi1995neural}, \cite{liestol1994survival}, \cite{brown1997use}), although these early models were not as deep or to say shallow enough to miss much of the underlying data patterns. The decades of \(2000s\) and \(2010s\) brought about an influx of deep learning and machine learning methods for survival analysis like the Random Survival forests(\cite{ishwaran2008random}), boosting based models(\cite{binder2008allowing}),etc. In addition to that, neural network based models with varying approaches were drawn to the implementation on survival analysis, like the cox based feedforward neural network model of DeepSurv in \cite{katzman2018deepsurv}, discrete time models like nnetsurvival in \cite{gensheimer2019scalable}, discrete time non parametric models like DeepHit \cite{lee2018deephit}, the deep learning adaptation of accelerated failure time model in \cite{norman2024deepaft}, etc. These models offer better heterogeneity and more flexibility, ensuring better performance by overcoming the bottlenecks left behind by the previous models. 
\paragraph{Multimodal Survival Analysis:} 
Modern survival analysis demands the inclusion of complex multimodal data for increased precision. In a survival study for decisive medical prognosis, the critical patients are often associated with different kinds of crucial data, this includes organized text data or clinial data capturing patient's basic health parameters (age, gender, blood pressure, sugar level, etc), pathological laboratory images and reports (such as ECG, MRI) and omics data (such as gene and microRNA expression and DNA methylation).  
In industrial settings, multimodal data consist of sensor readings, images, machine generated signals, operational logs, environmental measurements (such as temperature, humidity, and pressure) to monitor, analyze, and facilitate decision making, an example of such a model can be seen in \cite{al2019multimodal}.
Several Cox-based models are established among which around some of them have the capability to deal with multimodal data(example: DeepConvSurv \cite{7822579}). Several discrete time models also gained significance lately, these models offer more heterogeneity than Cox-based models, some examples of discrete time models which can also deal with multimodal data are Multisurv(\cite{vale2020multisurv}), DeepMMSA(\cite{wu2021deepmmsa}), MMsurv(\cite{yang2025mmsurv}), etc. The frequentist nature of these models raises the need for implementing self-aware Bayesian Multimodal model for survival analysis, but at the time of writing this paper, not much significant work has been done in this regard. 
\paragraph{Survival Analysis using Bayesian Neural Network:} The upsurge of the use of neural networks in estimation problems brought along the the idea of implementing the concept of Bayesian inference in it, thus giving birth to Bayesian neural networks. The use of Bayesian neural networks for survival analysis also saw simultaneous increase along the late \(1990s\) (\cite{bakker1999neural}), later on quite a number of research papers were released with the implementation of Bayesian neural networks in various kinds of survival data(on gastric cancer patients(\cite{kangi2018predicting}), Deep Bayesian survival analysis of rail useful lifetime(\cite{zeng2023deep}), etc), various significant studies include different approaches like Bayesian neural network based individual survival distribution(\cite{10158019}), Bayesian deep neural network for survival analysis using pseudo values(\cite{feng2021bdnnsurv}), etc. Bayesian neural networks have proved its significance in this high stake study of survival analysis due to its property of uncertainty quantification which is an effective way of getting far more reliable estimations. 


\paragraph{Vertical Federated Learning:} In many cases, hospitals or medical institutes may lack the facilities to conduct all the necessary diagnostic procedures for patients. Consequently, patient's data is distributed among multiple sources such as laboratories, hospitals, pathological labs and clinics. However, due to privacy concerns,  these institutes cannot share the raw data of patients with one another. In this scenario, to collaboratively train a model, Vertical Federated Learning is used where these individual clients(hospitals, pathological labs, clinics) train a local sub model to extract embedding representation from their corresponding data and then they send these embedding representations to a server for further aggregation, fusion of the embedding representations and for final prediction.
Vertical federated learning  has been an area of active research, with algorithms using linear regression(\cite{gascon2016privacy}), using tree based algorithms(\cite{wu2020privacy}), graph neural network based VFL(\cite{chen2020vertically}).     

\paragraph{Differential Privacy:} Privacy and security has been a growing concern with the emergence of collaborative model training environment like that of federated learning, subsequently several algorithms have been developed to ensure privacy which can be broadly classified into homomorphic encryption(\cite{yi2014homomorphic}), secure multi-party computation and differential privacy. Our interest in this paper lies in differential privacy(\cite{abadi2016deep}), which offers an edge over the other two aforementioned methods in terms of computational ease, but that comes with a tradeoff between privacy and model's performance. Several works and subsequent approaches have been developed for differential privacy, a hybrid differentially private vertical federated learning model is proposed in \cite{wu2020privacy}, another paper \cite{xu2021achieving} proposes the use of differential privacy in vertically separated multi-party learning. Differential privacy has been proved to be effective in privacy preservation along with computational ease, with the only downside being the tradeoff with model's performance, so it is important to find the appropriate amount of noise and build the model's framework  accordingly.

\section{Problem Statement}

In this paper, our main objective is to train a discrete time survival model that predicts interval wise risk probabilities of cancer patients. To do that we extract information from multimodal data, including clinical data, microRNA expression, and DNA methylation. Data of all the modalities may not be available in a single location, as it is often distributed across various sources, such as pathological laboratories, hospitals, clinics, etc, denoted as local clients for all patients. We have total $N$ number of patients. The data can be written as

\begin{equation*}
\mathcal{D} = \{(X_1, T_1, \mathcal{I}_1),  (X_2, T_2, \mathcal{I}_2), \ldots, (X_N, T_N, \mathcal{I}_N)\},   
\end{equation*}
    

where $X_i$ represents the covariate vector of $i$-th patient, which can be written as $X_i = \{X_{i1}, X_{i2}, \ldots , X_{id}\}$, and $d$ is the dimension of the feature vector. Suppose that the pair $y_i =(T_i, \mathcal{I}_i)$ represents the survival outcome, where $T_i$ denotes the observed time of an event of interest for the $i$-th patient and $\mathcal{I}_i$ is the event indicator, where $\mathcal{I}_i = 1$ indicates an observed event and $\mathcal{I}_i = 0$ indicates right censoring. \\
The dataset is further partitioned into $D_x$, representing the set of covariate vectors for N patients, and $D_y$, containing the $(T_i, \mathcal{I}_i)$ pairs, with both $D_x$ and $D_y$ linked through corresponding patient IDs. 

\begin{equation*}
    D_x = \{X_1 , X_2, ... , X_N\}, \quad
    D_y = \{y_1, y_2, \ldots, y_N\}.
\end{equation*}

We consider that the data for each patient is distributed in multiple locations. Let $M$ denote the number of clients. Each client contains a subset of covariates corresponding to a specific modality, and all clients share data for the same set of patients.
So, we represent our feature vector of $i$-th patient as

\[
X_i = \{X_i^{(1)}, X_i^{(2)}, ..., X_i^{(M)}\}. 
\]
\[
X_i^{(k)}  = \{X_{i1}^{(k)}, X_{i2}^{(k)}, \ldots, X_{id^k}^{(k)} \}, \quad \sum_{k=1}^M d^k = d.
\]

For extracting information from the multimodal data distributed in various locations, we construct a vertical federated learning (VFL) framework with $M$ clients and a central honest but curious server. Each client with data of a distinct modality for the same set of $N$ patients, train its  feature extractor model, say $f(\cdot)$ that maps the respective input data for $i$th patient's data for $k$th modality, say $X_i^{(k)}$ to a fixed size vector representation: \(E_{i}^{(k)}\) with dimension $d^{Embedding}$.
\[
E_i^{(k)} = f_k(X_i^{(k)}), \quad E_i^{(k)} = \{E_{i1}^{(k)}, E_{i2}^{(k)}, \ldots, E_{id^{Embedding}}^{(k)}\}.
\]




Clients send these representation vectors to the central server which contains the fusion layer, final prediction head, labels file and optimizer. For $i$th patient the central server receives the embeddings from all the clients represents as 

\[
E_i = \{E_{i}^{(1)}, E_{i}^{(2)}, \ldots, E_{i}^{(M)}\}.
\]


The fusion layer in the server fuses $E_i$ into a single compact vector $z_i$ for passing through the final prediction head to ultimately get interval wise survival probabilities. The labels file contains $(T_i, \mathcal{I}_i)$ pairs of the patients along with the corresponding patient IDs. This is used to calculate loss and subsequently the gradients which the server then sends back to the clients for locally updating their respecting sub model's parameters. The objective is to learn a multimodal survival model that estimates the conditional survival function

\[
S(t \mid z_i),
\]

where $S(t \mid \cdot)$ denotes the probability that the event time exceeds $t$, given all available modalities. Equivalently, the model may produce a risk score derived from this survival function.
We introduce differential privacy in our framework so that no attacker can infer sensitive information about patients. To guarantee privacy for patients, we add noise to the client embeddings before transmitting them to the server. Then the noisy embedding of $k$th client for $i$th patient is denoted as  
\[
E'_i = \{E_{i}^{'(1)}, E_{i}^{'(2)}, \ldots, E_{i}^{'(M)}\}, \quad E_i^{'(k)} = E_i^{(k)} + \xi^{(k)}_i. 
\]

where $\xi^{(k)}_i$ is the Gaussian noise added to the embedding of
client $k$ for patient $i$ with mean zero and variance $(\sigma_j^{k})^2$ of $j$-th element of the embedding. \\
Furthermore, survival analysis, especially in a setup for medical prognosis is a high stake study and demands a self aware model, so we leverage our VFL setup in a Bayesian framework, where all the client models, final prediction head and risk layer are Bayesian. So, the overall goal of the learning process is to jointly leverage all modalities under vertical data partitioning while satisfying two key requirements: (i) the ability to quantify predictive uncertainty, and (ii) protection against information leakage during federated training. The specific model architecture, learning strategy, uncertainty modeling, and privacy mechanisms are described in subsequent sections.

\section{Methodology}

\subsection{Survival Analysis}

Survival analysis studies the time until an event of interest occurs, such as death, disease progression, or system failure. It is widely used in medical prognosis, reliability analysis, and risk assessment. Let $T$ denote the continuous time-to-event random variable. The survival function is defined as
\[
S(t) = P(T > t) = 1 - F(t),
\]
where $F(t)$ is the cumulative distribution function of $T$. Survival analysis is often expressed in terms of the hazard function $h(t)$, which represents the instantaneous event rate at time $t$, given survival up to time $t$:
\[
h(t) = \lim_{\Delta t \to 0} \frac{P(t \le T < t + \Delta t \mid T \ge t)}{\Delta t}.
\]
The survival function can be written in terms of the hazard function as
\[
S(t) = \exp\left(-\int_0^t h(s)\, ds\right).
\]
For computational modeling, the continuous time axis is discretized into $\mathcal{J}$ disjoint intervals
\[
(0, t_1], (t_1, t_2], \ldots, (t_{\mathcal{J}-1}, t_\mathcal{J}].
\]
For subject $i$, the discrete-time hazard probability in interval $j$ is defined as
\[
h_{ij} = P(t_{j-1} \le T_i < t_j \mid T_i \ge t_{j-1}).
\]

The probability that subject $i$ survives beyond time $t_j$ is then
\[
S_{ij} = P(T_i > t_j) = \prod_{l=1}^{j} (1 - h_{il}).
\]

If subject $i$ experiences the event in interval $j$, the likelihood is given by
\[
\mathcal{L}_{ij} = h_{ij} \prod_{l=1}^{j-1} (1 - h_{il}).
\]
If the subject is right-censored at time $t_j$, the likelihood is
\[
\mathcal{L}_{ij} = \prod_{l=1}^{j} (1 - h_{il}).
\]

Let $j_i$ denote the interval corresponding to the observed or censored time of subject $i$. The log-likelihood for subject $i$ can be written as
\[
\mathcal{L}_i =
\mathcal{I}_i \left( \log h_{i j_i} + \sum_{l=1}^{j_i-1} \log (1 - h_{il}) \right)
+ (1 - \mathcal{I}_i) \sum_{l=1}^{j_i} \log (1 - h_{il}).
\]

Equivalently, the total log-likelihood over all $N$ subjects can be written as \cite{singer1993s}
\begin{equation}
\mathcal{L} = \sum_{i=1}^{N} \sum_{l=1}^{j_i}
\left[ \mathcal{I}_{il} \log h_{il} + (1 - \mathcal{I}_{il}) \log (1 - h_{il}) \right],
\label{likelihood_equation}
\end{equation}

Where $\mathcal{I}_{il}$ is a binary indicator function. $\mathcal{I}_{il}$ = 1, if $i$-th sample fails in $l$-th interval and $\mathcal{I}_{il}$ = 0, if does not fail. 

In our senario, time length is taken upto the maximum observed time $T_i$ among all the patients. The time length is divided into $p$ equal length interval. So, the total log-likelihood is written as 

\begin{equation}
\mathcal{L} = \sum_{i=1}^{N} \sum_{l=1}^{p}
\left[ \mathcal{I}_{il} \log h_{il} + (1 - \mathcal{I}_{il}) \log (1 - h_{il}) \right].
\label{lik_eq}
\end{equation}


In our setting, the observed time-to-event depends on multimodal covariates. To estimate the $h_{il}$ and corresponding survival function by incorporating the effects of these covariates, we adopt a multimodal survival analysis framework that extracts and fuses information from multimodal data.

\subsection{Multimodal Survival Analysis}

Recent advances in biomedical research have led to the emergence of high dimensional multimodal data, motivating their integration into survival analysis. Existing multimodal survival models, such as MultiSurv, DeepMMSA, and MMsurv, employ modality specific sub networks to extract features that are subsequently fused for risk prediction. However, these approaches are predominantly frequentist and lack principled uncertainty quantification.
In Figure \eqref{fig:placeholder}, we present the architecture of the proposed BVFLMSP model. In this framework, three modalities are considered as passive clients in a vertical federated learning setting, corresponding to clinical data, DNA methylation, and microRNA respectively. Each sub model is trained locally in respective clients, and the final embeddings are transmitted to the server. At the server, these embeddings are fused to pass through the final prediction head which further passes through the risk layer to get the final predictive outcome.

In problem statement section, we have explained the expression of multimodal data with $M$ modalities.

\[
X_i = \{X_i^{(1)}, X_i^{(2)}, ..., X_i^{(M)}\}. 
\]

Each modality is processed by a dedicated submodel $f_k(\cdot)$ acting as a feature extractor:

\[
\mathbf{E}^{(k)}_i = f_k(\mathbf{X}^{(k)}_i), \quad k = 1, \dots, M.
\]
For patient $i$, the modality-specific representations are collected as

\[
E_i = \{E_{i}^{(1)}, E_{i}^{(2)}, \ldots, E_{i}^{(M)}\}.
\]

A fusion mechanism, for our model Attention mechanism, $Attn(\cdot)$ aggregates these representations into a single latent embedding vector by assigning attention weights \(\alpha_k\) to all the individual embedding vectors \(E_{i}^{k}\) and combining them as:
\[
\mathbf{z}_i = Attn(E_i)=\underset{k=1}{\overset{M}{\sum}}\alpha_kE_{i}^{(k)}.
\]

The fused representation is passed through a prediction head to obtain discrete-time hazard probabilities:
\[
\mathcal{F}(\mathbf{z}_i) = \mathbf{h}_i,
\quad
\mathbf{h}_i = (h_{i1}, h_{i2}, \dots, h_{ip}).
\]
where $h_{ij}$ denotes the hazard probability for patient $i$ in the $j$-th time interval, and $p$ is the number of discretized time intervals.
The survival loss is defined by the negative log-likelihood loss as defined in \eqref{lik_eq}.
While existing multimodal survival models follow this general framework, they are largely frequentist in nature and do not explicitly model predictive uncertainty. To address this limitation, we propose a Bayesian multimodal survival model and further extend it to a vertical federated learning setting.

\subsection{Bayesian Neural Networks}

Bayesian neural networks (BNNs) are grounded in Bayesian inference, where probability is interpreted as a measure of belief rather than long-run frequency, as in the frequentist paradigm. While frequentist inference aims to provide frequency-based guarantees, Bayesian inference focuses on expressing and updating subjective beliefs through probability distributions \cite{box2011bayesian}. 

Let $\{\theta_1, \dots, \theta_M\}$ are the local client parameters and $\theta_s$ denote the parameters of the server-side model, including the fusion mechanism and the final survival prediction head.
\[\Phi = \{\theta_1, \dots, \theta_M, \theta_s\}.\]
In Bayesian inference, model parameters $\Phi$ are treated as random variables and assigned prior distributions $p(\Phi)$ encoding prior beliefs. Given a dataset, Bayes' theorem is used to compute the posterior distribution over parameters:
\[
p(\Phi \mid \mathcal{D}) = \frac{p(\mathcal{D} \mid \Phi)\, p(\Phi)}{p(\mathcal{D})},
\]
where $p(\mathcal{D}) = \int p(\mathcal{D} \mid \Phi) p(\Phi)\, d\theta$ is the marginal likelihood.

Uncertainty in parameter estimates is quantified using \emph{credible sets}. A $95\%$ credible set $C$ satisfies
\[
\int_{C} p(\Phi \mid \mathcal{D})\, d\Phi = 0.95.
\]

Bayesian neural networks \cite{neal2012bayesian} extend this framework to deep learning by placing prior distributions over network weights and biases and learning their posterior distributions given data. Unlike deterministic neural networks that yield point estimates, BNNs represent uncertainty in model parameters and predictions.

Let $p(\Phi)$ denote the prior over network parameters. The posterior distribution is given by

\[
p(\Phi \mid \mathcal{D}) =
\frac{p(\mathcal{D} \mid \Phi)\, p(\Phi)}
{\int p(\mathcal{D} \mid \Phi')\, p(\Phi')\, d\Phi'},
\]

To make predictions for a new input $X_{new}$, posterior predictive inference is performed by sampling parameters from the posterior distribution.$m$ is total number of samples generated from the posterior distribution. Specifically, for $i = 1, \dots, m$,
\[
\phi_i \sim p(\Phi \mid \mathcal{D}), \quad
y_i = \mathcal{F}_{\phi_i}(X_{new}),
\]
where $\mathcal{F}_{\phi}(\cdot)$ denotes the forward pass of the neural network with parameters $\phi$.

The collection of predictions $\mathcal{Y} = \{y_i\}_{i=1}^{m}$ forms a confidence set for the output value for the corresponding input data of the particular subject. This confidence set accounts for the uncertainty in the model's predictions.


\textbf{Note 1}(Bayesian regularization and robustness).
\textit{Bayesian models often exhibit improved robustness \cite{murphy2012machine} compared to unregularized frequentist models due to the implicit regularization introduced by prior distributions on model parameters.}

In standard non-Bayesian learning, model parameters are commonly estimated using maximum likelihood estimation (MLE),
\[
\widehat{\Phi}_{\mathrm{MLE}} = \arg\min_{\Phi} \left[ - \log p(\mathcal{D} \mid \Phi) \right].
\]

Although MLE is asymptotically unbiased under mild regularity conditions \cite{lehmann1998theory}, the entire error of the model becomes dependent on the estimator variance which can be harmful under finite-sample or noisy settings, leading to unstable parameter estimates.

In contrast, Bayesian models impose prior distributions over model parameters. Assuming independent zero-mean Gaussian priors for $M$ clients and central server,
\[
P(\theta_k) = \mathcal{N}(\theta_k \mid 0, \sigma^2)
= \frac{1}{\sqrt{2\pi\sigma^2}} \exp\!\left(-\frac{\theta_k^2}{2\sigma^2}\right), \quad
\theta_k \in \Phi,
\]
the joint prior becomes
\[
P(\Phi) = \prod_k \mathcal{N}(\theta_k \mid 0, \sigma^2).
\]
Taking logarithms,
\[
\log P(\Phi) = c - \frac{1}{2\sigma^2} \sum_k \theta_k^2,
\]
where \(c\) is a constant independent of \(\theta_k\).

Bayesian learning seeks the Maximum A Posteriori (MAP) estimate \cite{murphy2012machine}:
\[
\widehat{\Phi}_{\mathrm{MAP}} = \arg\max_{\Phi} P(\Phi \mid \mathcal{D})
= \arg\max_{\Phi} \big[ \log P(\mathcal{D} \mid \Phi) + \log P(\Phi) \big],
\]
which is equivalently written as
\begin{equation}
\widehat{\Phi}_{\mathrm{MAP}}
= \arg\min_{\theta}
\left[
- \log P(\mathcal{D} \mid \Phi)
+ \frac{1}{2\sigma^2} \sum_k \theta_k^2
\right].
\label{bayesian robustness}
\end{equation}
This objective corresponds to L2-regularized maximum likelihood estimation,
where the Gaussian prior induces weight decay. The regularization term penalizes large parameter magnitudes, reducing sensitivity to noise and perturbations in the observed data. While this introduces bias into the estimator, it substantially reduces variance. Under the bias variance trade off, this variance reduction often leads to improved stability, better calibration, and enhanced robustness in practice, particularly in noisy, small sample, vertically partitioned, or privacy constrained learning settings, as encountered in federated survival analysis with differential privacy. In Bayesian Neural Networks, the posterior distribution over model parameters is often intractable, so, using MAP helps in maintaining robustness, but fails to capture model uncertainty. To solve this problem, Kullback-Leibler(KL) divergence regularizer\cite{blundell2015weightuncertaintyneuralnetworks} is used, which helps in uncertainty quantification by learning a posterior distribution \(q(\Phi)\) for the model weights \(\Phi\), while keeping the model robust by forcing \(q(\Phi)\) to be not too far from the prior distribution \(p(\Phi)\).

\subsection{Kullback-Leibler Divergence Regularization}

We employ variational inference to approximate the intractable posterior distribution over model parameters. Consequently, the training objective includes a Kullback--Leibler (KL) divergence regularization term that penalizes deviation of the variational posterior from a prior distribution.

The KL divergence between an approximate posterior $q(\Phi)$ and a prior $p(\Phi)$ is defined as
\[
\mathrm{KL}(q(\Phi)\|p(\Phi))
=
\int q(\Phi)\log\frac{q(\Phi)}{p(\Phi)}\,d\Phi
=
\mathbb{E}_{q}[\log q(\Phi)] - \mathbb{E}_{q}[\log p(\Phi)].
\]

Assume a Gaussian prior over model parameters, $p(\theta_k) = \mathcal{N}(\mu_0, \sigma_0^2)$, and a Gaussian variational posterior, $q(\theta_k) = \mathcal{N}(\mu, \sigma^2)$. The KL divergence admits a closed-form expression:
\[
\mathrm{KL}\big(\mathcal{N}(\mu, \sigma^2)\,\|\,\mathcal{N}(\mu_0, \sigma_0^2)\big)
=
\log\frac{\sigma_0}{\sigma}
+
\frac{\sigma^2 + (\mu - \mu_0)^2}{2\sigma_0^2}
-
\frac{1}{2}.
\]


Thus, the regularization term penalizes large values of $(\mu - \mu_0)^2$, discouraging the posterior mean from drifting too far away from the prior mean and preventing excessively large model weights, which contributes to robustness. In addition, the term $\log(\sigma_0/\sigma)$ penalizes over-confident posterior distributions, as it diverges when $\sigma \to 0$. Conversely, the term $\sigma^2/(2\sigma_0^2)$ penalizes excessively uncertain (under-confident) posteriors, since it grows with increasing $\sigma^2$. Together, these terms enforce a balance between over-confidence and under-confidence, yielding an optimal level of uncertainty while maintaining model robustness. We add this KL divergence regularizer term to our model's overall loss function which is discussed in the next section.

\subsection{Loss Function}
The loss function of our model is combination of discerete time survival negative log-likelihood loss, KL divergence regularizer and auxiliary loss given by: 
\begin{equation}
    \mathcal{L}_{\text{total}}
    =
    \mathcal{L}_{\text{surv}}
    +
     \frac{1}{N_t} \sum_{\Phi} \mathrm{KL}\big(q(\Phi)\,\|\,p(\Phi)\big)
    +
    0.05\, \mathcal{L}_{\text{auxiliary}}.\label{total_loss}
\end{equation}
where \(N_t\) denotes the number of training samples, \(q(\Phi)\) is the variational posterior over model parameters, and \(p(\Phi)\) is the prior distribution.

The auxiliary loss encourages alignment between latent feature representations obtained from different modalities. For a pair of latent feature vectors \((z_i, z_j)\) corresponding to the same patient but extracted from two different modalities, the auxiliary loss is defined as
\[
\mathcal{L}_{\text{auxiliary}}(z_i, z_j)
=
1 - \cos(z_i, z_j),
\]
where \(\cos(\cdot, \cdot)\) denotes cosine similarity. The total auxiliary loss is computed as
\[
\mathcal{L}_{\text{auxiliary}}
=
\frac{1}{N_{\text{pairs}}}
\sum_{i<j}
\mathcal{L}_{\text{auxiliary}}(z_i, z_j),
\]
where the summation is taken over all modality pairs for each patient, and \(N_{\text{pairs}}\) denotes the total number of such pairs.

This auxiliary term promotes modality-invariant representations by encouraging feature embeddings from different modalities to lie close in the shared latent space, thereby facilitating more effective multimodal fusion.

\subsection{Vertical Federated Learning}

Vertical Federated Learning (VFL) \cite{liu2024vertical} is a federated learning paradigm in which multiple parties collaboratively train a model by holding different feature subsets corresponding to the same set of samples. In contrast to horizontal federated learning, data samples are aligned by identity across clients, while features are vertically partitioned. Raw features are never shared across parties. 

\begin{figure}[h]
    \centering
    \includegraphics[width=0.8\linewidth]{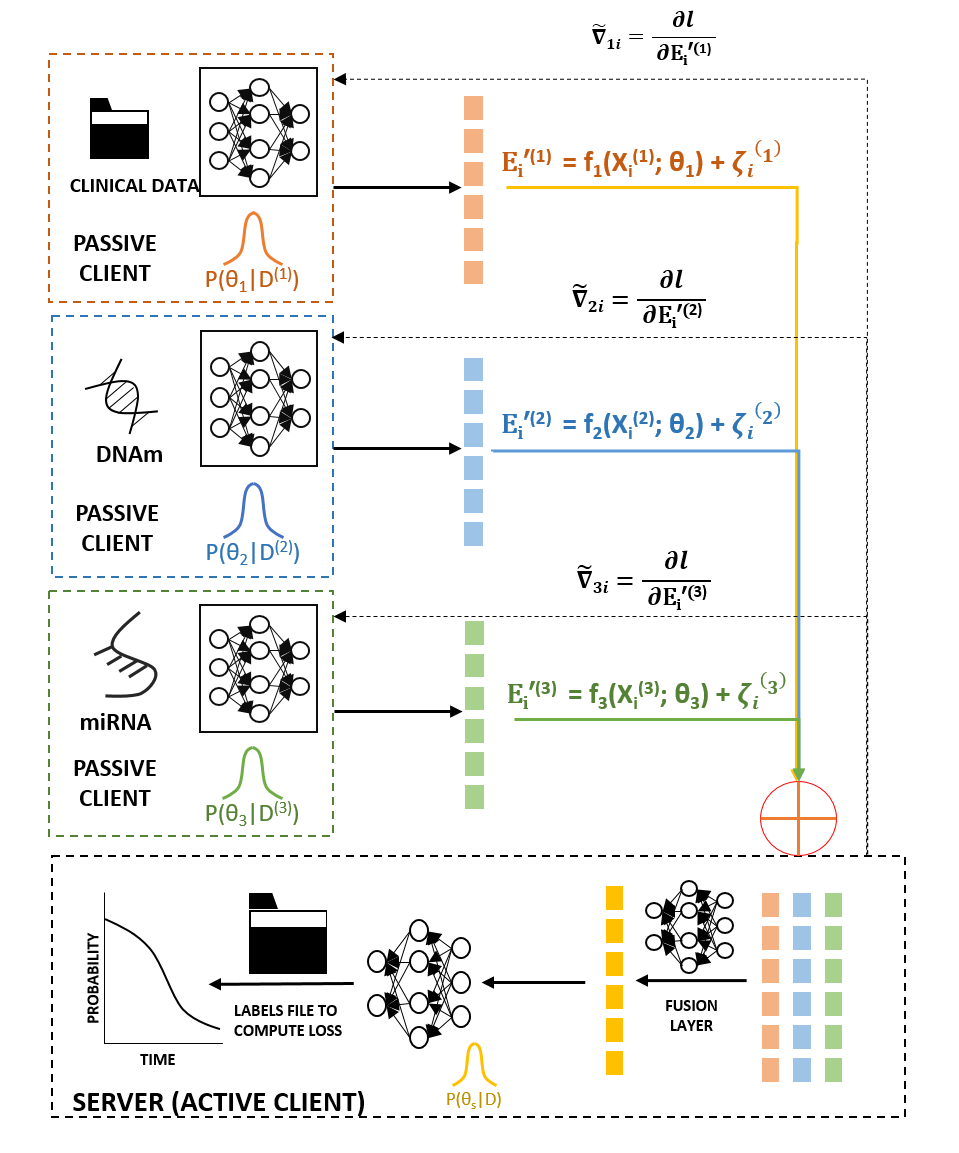}
    \caption{Architecture of BVFLMSP.}
    \label{fig:placeholder}
\end{figure}

In our framework, each data modality is treated as a passive client, while the central server acts as an active client, holding the survival labels as well as the fusion module and the final prediction head. Each client maintains a modality specific feature extractor that maps local features to latent representations. During training, clients transmit only these intermediate representations to the server, which aggregates them to produce survival predictions. Gradients with respect to the embeddings are then returned to the corresponding clients to enable end-to-end training. We further impose prior distributions over model parameters, resulting in a Bayesian vertical federated learning formulation. Figure~\ref{fig:placeholder} illustrates the  architecture of the proposed Bayesian VFL framework, BVFLMSP with three modalities. \\
At each training iteration, client \(k \in \{1,\dots,M\}\) computes the output of its local model on the input of $i$th patient and sends the embedding $E_i^{(k)}$ to the central server. Then fusion is performed on the embeddings and the fused layer $z_i$ passes through the final prediction head.
Let $\ell(\cdot,\cdot)$ denote the loss function \eqref{total_loss}. 
We consider the optimization problem,
\begin{equation*}
\min_{\theta_1, \dots, \theta_M, \theta_s}
\frac{1}{N} \sum_{i=1}^{N}
\ell \Big(
\mathcal{F}\big(z_i),
y_i
\Big).
\label{eq:objective}
\end{equation*}

The central server computes the gradient of the loss with respect to its parameters $\nabla_{\theta_s} \ell = \frac{\partial \ell}{\partial \theta_s}$, and updates \(\theta_s\) accordingly. It also computes the gradients of the loss with respect to each client embedding:
\[
\frac{\partial \ell}{\partial E^{(k)}}, \quad \forall k \in \{1,\dots,M\},
\]
and sends these gradients back to the corresponding clients. Upon receiving \(\frac{\partial \ell}{\partial E^{(k)}}\), each client \(k\) computes the
gradient of the loss with respect to its local model parameters:
\[
\nabla_{\theta_k} \ell
=
\frac{\partial \ell}{\partial \theta_k}
=
\sum_{i}
\frac{\partial \ell}{\partial E_i^{(k)}}
\frac{\partial E_i^{(k)}}{\partial \theta_k}.
\]
Each client then updates its local model parameters \(\theta_k\) accordingly. This
procedure is repeated iteratively until convergence.

In vertical federated learning, there is a risk of information leakage, either through the embeddings transmitted from the client to the server or through the gradients sent from the server back to the client. In the following section, we provide one such attack structure where the server tries to breach client privacy using their embedding outputs.

\paragraph{Attack Structure: }
We consider a feature reconstruction attack under an honest but curious server assumption, following optimization based data recovery attacks such as \cite{jin2021cafe}. The server strictly follows the vertical federated learning (VFL) protocol but attempts to infer clients’ private input features from the communicated intermediate representations. We adopt a black-box threat model in which the server has no access to the clients’ raw data or private model parameters, but observes the embedding outputs exchanged during training.

The server is assumed to have access to an auxiliary public dataset $D_{\text{public}}$ drawn from a distribution similar to the clients’ private data. Using this public data, the server trains a shadow VFl model with separate feature extractor models that mimic the client-side feature extractors in the VFL setting. Let $f_c(\cdot)$ denote the client model and $f_s(\cdot)$ denote the shadow model. The objective of the server is to learn $f_s$ such that
\[
f_s(\cdot) \approx f_c(\cdot),
\]
where the alignment is performed by matching the distribution of embeddings observed during VFL training using public samples as inputs.

For samples $X'_i \sim D_{\text{public}}$, the server obtains input embedding pairs by forwarding the public data through the shadow model:
\[
E'_i = f_s(X'_i).
\]
Using these pairs $(E'_i, x'_i)$, the server then trains a decoder network $g(\cdot)$ to reconstruct the original input features from the embeddings. The decoder is trained by minimizing the reconstruction loss
\[
\mathcal{L}_{\text{decoder}}(\theta)
=
\mathbb{E}_{x' \sim D_{\text{public}}}
\left[
\left\| g(E') - x' \right\|_2^2
\right],
\]
where $w$ denotes the decoder parameters, and the optimal parameters are obtained as
\[
\theta^* = \arg\min_\theta \mathcal{L}_{\text{decoder}}(\theta).
\]

After training, the decoder parameters are fixed. During the actual VFL training process, when the server receives true client embeddings $E_i = f_c(X_i)$, it forwards these embeddings through the trained decoder $g(\cdot)$ to obtain reconstructions $\hat{X}_i = g(E_i)$ of the clients’ private input features. This enables the server to recover approximations of the raw inputs, constituting a privacy breach. For ensuring privacy protection against this attack structure, we need a solid defense mechanism, the description of which is provided in the next section.

\paragraph{Defense Mechanism: }
Most existing studies on differential privacy in federated learning focus on perturbing either (i) the gradients transmitted from the server to clients \cite{ranbaduge2022differentiallyprivateverticalfederated} or (ii) the model parameters shared by clients for aggregation \cite{he2023clustered}. However, in SplitNN-based vertical federated learning, these mechanisms are insufficient for protecting client feature privacy.

In our setting, clients do not transmit model parameters to the server; instead, they communicate intermediate embedding outputs. Consequently, perturbing client-side model weights provides no privacy protection, as these parameters are never observed by the server. Similarly, server-side gradient perturbation is inadequate because the server receives clean embedding outputs during the early rounds of training. These unperturbed embeddings can be exploited by a malicious server to train a shadow model together with a decoder that learns to invert embeddings back to raw input features, enabling feature recovery attacks.

To address this fundamental vulnerability, we propose a client-side embedding perturbation mechanism in which differential privacy noise is directly added to the embedding outputs before transmission to the server. This ensures that the server never observes clean client representations, even in the initial training rounds, preventing it from learning a reliable inverse mapping from embeddings to raw input features and thereby reducing the risk of feature reconstruction and data recovery attacks. Furthermore, client side embedding perturbation eventually results in noisy gradients as shown in Equation \eqref{eq:noisy_grad}.

\subsection{Differential Privacy}

Differential privacy \cite{abadi2016deep} allows us to maintain privacy in the data. In our model, differential privacy is incorporated by adding noise to the input features of each client before training, thus providing privacy guaranties at the feature level and ensures that the contribution of any individual data point remains private while preserving the effectiveness of the collaborative training process. 

\begin{definition}
A randomized mechanism \cite{tran2023privacy} \( \mathcal{M} : D \to \mathbb{R} \) with domain
\( D \) and range \( \mathbb{R} \) said to satisfy \((\epsilon,\delta)\)-differential
privacy if, for any two adjacent datasets, \( \mathcal{D}, \mathcal{D'} \in D \) that differ in atleast one data sample, and for any subset
of outputs \( S \subseteq \mathbb{R} \), it holds that
\[
\Pr\bigl[ \mathcal{M}(\mathcal{D}) \in S \bigr]
\le
e^{\epsilon} \Pr\bigl[ \mathcal{M}(\mathcal{D'}) \in S \bigr] + \delta.
\]
\end{definition}

Where $\epsilon$ is the privacy budget that controls the privacy strength and $\delta$ represents the probability that the privacy loss exceeds $\epsilon$. A small $\epsilon$ provides stronger privacy protection, while a larger $\epsilon$ allows more accurate result but weaker privacy.

Following \cite{abadi2016deep} we introduce the following theorem. For the theorem, we take into account the following assumptions:

\textit{\textbf{Assumption 1:} We clip the embeddings by a constant C, consequently the embeddings strictly doesn't leave the clients with the \(L_2\) norm greater than C.\\
\textbf{Assumption 2:} The batches of data used in training are selected uniformly at random with replacement such that the probability of a particular data point to be in the batch is p, where p = (batch size)/(Training data set size) is sufficiently small.}\\

We have used the following lemmas from \cite{abadi2016deep} to facilitate the proof of our theorem:
\begin{lemma}
    In a Gaussian mechanism \(\mathcal{Q}\) with noise scale \(\sigma\) and batch sampling \(p\), privacy loss's log-moment (\(\alpha_{\mathcal{Q}_{i}}(\lambda)\)) is bounded as:
\begin{equation}
 \alpha_{\mathcal{Q}_{i}}(\lambda)\leq \frac{p^2(\lambda)(\lambda+1)}{(1-p)\sigma^2}+O\left(\frac{p^3}{\sigma^3}\right).
 \end{equation}
 For all i \(\in\) 1(1)\(\tau\).
\end{lemma}
\begin{lemma}
    \textbf{Composability:} Suppose that a mechanism \(\mathcal{Q}\)  consists of sequence of adaptive mechanisms \(\mathcal{Q}_1,\mathcal{Q}_2,...,\mathcal{Q}_{k}\) where \(\mathcal{Q}_i=\underset{j=1}{\overset{i-1}{\prod}}\mathbb{R}_j\times D\rightarrow\mathbb{R}_i\). Then for any \(\lambda\).
\begin{equation}
    \alpha_{\mathcal{Q}}(\lambda)\leq \underset{i=1}{\overset{k}{\sum}}\alpha_{\mathcal{Q}_i}(\lambda).
    \label{composability}
\end{equation}

\end{lemma}
\begin{lemma}
    \textit{\textbf{Tail bound:}} For any \(\epsilon>0\), the probability \(\delta\) that the privacy loss exceeds \(\epsilon\) is bounded by: 
\begin{equation}
\delta = \underset{\lambda}{min}\exp(\alpha_{\mathcal{Q}^{\tau}}(\lambda)-\lambda \epsilon).
\label{tail bound}
\end{equation}
\end{lemma}

\begin{theorem}
A mechanism \(\mathcal{Q}\)  executes the client level models over  \(\tau\)  training iterations. The subsequent embedding outputs are bounded by a constant C and we add Gaussian noise sampled from N(0, \(\sigma^2d^2\)) to the embedding outputs at client level in each iteration. Thus, if there exist positive constants \(c_1\) and \(c_2\), then the mechanism \(\mathcal{Q}\) is ( \(\epsilon,\delta\)) differentially private for any \(\epsilon<c_1\sigma^2\tau\) with \(\delta>0\) if the standard deviation of noise is characterized by:

\begin{align}
    \sigma \geq c_2\frac{p\sqrt{\tau log(\frac{1}{\delta})}}{\epsilon},
    \label{eq:sigma}
\end{align}
where p is the probability by which batches of data are sampled for training.
\end{theorem}
\begin{proof}
    Let the feature extractor model of the ith client be represented as \(f_i(.)\), for training step \(t\in \{1,2,...,\tau\}\) the clients process a batch of data \(\mathcal{B}\) selected at random with the probability \(p\), and outputs:
\(o_t=f(\mathcal{B})+\epsilon_t, \)
where, \(f(\mathcal{B})\) is clipped and thus \(||f(\mathcal{B})||_{2}^{2}\leq C\), \(\epsilon_t\sim \mathcal{N}(0,\sigma^2C^2I).\)

Privacy loss is given by \(c(o)\), which is a random variable and is formulated as: 
\[c(o)=ln\left(\frac{P[\mathcal{Q}(\mathcal{D})=o]}{P[\mathcal{Q}(\mathcal{D'})=o]}\right).\]

We will analyze our privacy loss using moments accountant method introduced in \cite{abadi2016deep}, where instead of tracking \((\epsilon,\delta)\) directly, we use the log-moment generating function(\(\alpha (\lambda)\)) of the privacy loss random variable \(c(o)\), which is given as:
\[\alpha_{\mathcal{Q}}(\lambda)\overset{\Delta}{=}ln\mathbb{E}_{o\sim\mathcal{Q}(\mathcal{D})}[exp(\lambda\cdot c(o))].\]
Now, our defense mechanism \(\mathcal{Q}\) runs over \(\tau\) steps and in every step the random noise is sampled from Gaussian distribution independent from each other, so, following lemma 2 we get: 
\[\alpha_{\mathcal{Q}^{\tau}(\mathcal{D})}(\lambda)\leq \underset{i=1}{\overset{\tau}{\sum}}\alpha_{\mathcal{Q}_{i}}(\lambda)=\tau\cdot\alpha_{\mathcal{Q}_{i}}(\lambda).\]
Then from lemma 1, we get that the privacy loss's log-moment (\(\alpha_{\mathcal{Q}_{i}}(\lambda)\)) is bounded as:
\[\alpha_{\mathcal{Q}_{i}}(\lambda)\leq \frac{p^2(\lambda)(\lambda+1)}{(1-p)\sigma^2}+O\left(\frac{p^3}{\sigma^3}\right),\]
where, \(\sigma\) is the noise scale and \(p\) is the batch sampling probability. Now, we have assumed \(p\) to be significantly small, so we can ignore the order term as it becomes negligible. Thus we can write this asymptotically as:
\[\alpha_{\mathcal{Q}_{i}}(\lambda)=\frac{ p^2\lambda^2}{\sigma^2}.\]
Thus we compute the total privacy loss over \(\tau\)  steps as: 
\[\alpha_{\mathcal{Q}^{\tau}}(\lambda)\leq\frac{\tau p^2(\lambda^2)}{\sigma^2}.\]
Now, from lemma 3, we get that for any \(\epsilon>0\), the probability \(\delta\) that the privacy loss exceeds \(\epsilon\) is bounded by: 
\[\delta = \underset{\lambda}{min}\exp(\alpha_{\mathcal{Q}^{\tau}}(\lambda)-\lambda \epsilon).\]

\begin{equation}
\delta=\underset{\lambda}{min}\exp\left(\frac{\tau p^2\lambda^2}{\sigma^2}-\lambda \epsilon\right)
\label{privacy loss probab}
\end{equation}

So, to get the minimum \(\delta\), we minimize the element with respect to \(\lambda\) by differentiating with respect to \(\lambda\) and set it equal to 0:
\[\frac{d}{d\lambda}\left(\frac{\tau p^2\lambda^2}{\sigma^2}-\lambda \epsilon\right)=\frac{2\tau p^2\lambda}{\sigma^2}-\epsilon=0\]
\[\lambda_{optimum}=\frac{\epsilon\sigma^2}{2\tau p^2}\]
Now, substituting this value in \ref{privacy loss probab} we get:
\[\delta=\exp\left(\frac{\tau p^2\left(\frac{\epsilon^2\sigma^4}{4\tau^2p^4}\right)}{\sigma^2}- \frac{\epsilon^2\sigma^2}{2\tau p^2}\right)\]
\[\delta=\exp\left(-\frac{\epsilon^2\sigma^2}{4\tau p^2}\right)\]
Now, we require the probability that privacy loss exceeds \(\epsilon\) to be at max \(\delta\). So, the requirement for \((\epsilon,\delta)\) differential privacy becomes:
\[\delta\geq\exp\left(-\frac{\epsilon^2\sigma^2}{4\tau p^2}\right)\]
\[ ln(\delta)\geq-\frac{\epsilon^2\sigma^2}{4\tau p^2}\]
\[ ln\left(\frac{1}{\delta}\right)\leq\frac{\epsilon^2\sigma^2}{4\tau p^2}\]
\[\sigma\geq\frac{2p\sqrt{\tau ln\left(\frac{1}{\delta}\right)}}{\epsilon}.\]
Hence, proved.

\end{proof}

\section{Convergence Analysis}\label{convergence analysis}


In this section, we provide the convergence analysis of BVFLMSP algorithm \eqref{alg:dp_attn_bayesian_vfl} in terms of optimality gap with respect to the optimal value. Optimality gap measures how close the obtained solution is to the optimal solution. To establish convergence of BVFLMSP algorithm, we show that the expected optimality gap at each epoch is bounded and decreases as the number of epoch increases.
To establish this result, we need to consider the following assumptions:

\textbf{Assumption 4 ($\alpha$-Strong Convexity \cite{shi2023vertical, gai2025differentially}):}
The function $J(\cdot): \mathbb{R}^d \rightarrow \mathbb{R}$ is
$\alpha$-strongly convex, where $\alpha$ is constant and $\alpha > 0$,  i.e., for all $x_1, x_2 \in \mathbb{R}^d$,
\begin{equation}
J(x_2) \ge J(x_1) + \nabla J(x_1)^\top (x_2 - x_1)
+ \frac{\alpha}{2} \|x_2 - x_1\|^2 .
\label{assump:strong_convex}
\end{equation}

\textbf{Assumption 5 ($\beta$-Smoothness \cite{shi2023vertical, gai2025differentially}):}
The function $J(\cdot)$ is $\beta$-smooth, where $\beta$ is constant and $\beta > 0$, i.e., for all
$x_1, x_2 \in \mathbb{R}^d$,
\begin{equation}
J(x_2) \le J(x_1) + \nabla J(x_1)^\top (x_2 - x_1)
+ \frac{\beta}{2} \|x_2 - x_1\|^2 .
\label{assump:smoothness}
\end{equation}


\begin{theorem}[Convergence of BVFLMSP]
\label{thm:main}
Under Assumptions 4 and 5, with the learning rate $\eta = 1/\beta$, the expected optimality gap of BVFLMSP is upper bounded by equation \eqref{eq:convergence_inequality}. $L_E$ is constant. $\alpha$ and $\beta$ are constants and $\alpha > 0$, $\beta > 0$. L is the total number of epochs. $J(\cdot)$ is the loss function. N is the total number of samples.

\begin{align}
\mathbb{E}\!\left[ J(\Phi^{L}) - J(\Phi^*) \right]
\le\;&
\left(1 - \frac{\alpha}{\beta}\right)^{L}
\mathbb{E}\!\left[ J(\Phi^{(0)}) - J(\Phi^*) \right] \nonumber \\
&+
\frac{256*\sigma^2*M*L_E*\beta}{N*\alpha}
*\left[1 - \left(1-\frac{\alpha}{\beta}\right)^L\right].
\label{eq:convergence_inequality}
\end{align}


\end{theorem}

\begin{proof}

See Appendix~\ref{app:proof} for the proof.

\end{proof}




\begin{algorithm}[h]
\caption{BVFLMSP}
\label{alg:dp_attn_bayesian_vfl}
\begin{algorithmic}[1]
\State \textbf{Input:} Dataset $\mathcal{D} = \{X^{(k)}\}_{k=1}^M$ (M: Number of clients), Labels $Y = (T, \mathcal{I})$, T: survival time, \(\mathcal{I}\): event status indicator.
\State \textbf{Input:} Hyperparameters: Learning rate $\eta$, KL weight $\beta$, Clipping Norm $C$, Noise Multiplier $\sigma$.
\State \textbf{Initialize:} Client parameters $\theta_k$, Server parameters $\theta_s$ (including Attention weights).

\For{epoch $= 1, \dots, L$}
    \For{each mini-batch $\mathcal{B} \in \mathcal{D}$}
        \State \textbf{// Client Side (Forward with DP) }
        \For{each client $k = 1, \dots, M$}
            \State Sample weights $W^{(k)} \sim q_{\theta_k}$, ($q_{\theta_k}$ : Posterior distribution of parameter $\theta_k$).
            \State Compute raw embeddings: $\tilde{E}^{(k)}_b = f_k(x^{(k)}_b; W^{(k)})$.
            
            \State \textit{DP Step 1: L2 Norm Clipping}
            \State Calculate norms: $|| \tilde{E}^{(k)}_b ||_2$.
            \State $E^{(k)}_b = \tilde{E}^{(k)}_b \cdot \min\left(1, \frac{C}{|| \tilde{E}^{(k)}_b ||_{2} + \epsilon}\right)$.
            \State Here $\epsilon=0.000001$ added to ensure non-zero denominator
            \State \textit{DP Step 2: Noise Injection}
            \State Sample noise $\xi \sim \mathcal{N}(0, \sigma^2 C^2 I)$.
            \State Perturb embeddings: $E'^{(k)}_b = E^{(k)}_b + \xi^{(k)}$.
            
            \State Compute local KL: $\mathcal{L}_{\text{KL}}^{(k)} = \text{KL}[q_{\theta_k} || p]$, p: prior distribution.
            \State Send perturbed embedding $E'^{(k)}_b$ to Server.
        \EndFor
        
        \State \textbf{// Server Side (Forward \& Loss) }
        \State Receive $\{E^{(1)}_b, \dots, E^{(M)}_b\}$.
        \State Sample server weights $W^{(\text{top})} \sim q_{\theta_s}$. Here $q_{\theta_s}$: posterior distribution of $W^{(\text{top})}$.
        \State \textbf{Fusion Mechanism (Attention):}
        \State Stack embeddings: $\mathbf{E}_b = \text{Stack}(E^{(1)}_b, \dots, E^{(M)}_b)$
        \State Fuse the embeddings using attention mechanism:
        \State $z_b = \text{Attention}(\mathbf{E}_b; \phi_{\text{attn}})$
        
        \State Predict hazard vector $\hat{h}$: $\hat{h}_b =\sigma(g(z_b; W^{(\text{top})}))$.
        
        \State \textbf{// Compute Discrete-Time Log-Likelihood:}
        \State Compute survival loss using the equation as in \ref{lik_eq}
        \State Compute Total Loss:
        \State Compute total loss using the equation as in \ref{total_loss}
        \State \textbf{// Backward Pass }
        \State Update Server $\theta_s \leftarrow \theta_s - \eta \nabla_{\theta_s} \mathcal{L}$.
        \State Compute gradients w.r.t fused embedding $z_b$.
        \State Backpropagate through Attention layer to get $\nabla_k = \frac{\partial \mathcal{L}}{\partial E^{(k)}_b}$.
        \State Send split gradients $\nabla_k$ to Clients.
        \State Clients update $\theta_k$ using $\nabla_k$ and local KL gradients.
    \EndFor
\EndFor
\end{algorithmic}
\end{algorithm}

\section{Experiment with real life dataset}
In the following sections, we describe our experiment with a real life dataset in establishing a comparative study between a centralized model and a vertical federated learning model in a survival study setup.


\subsection{Dataset Description}
The data used in this study were obtained from the NCI Genomic Data Commons (GDC) portal \cite{jensen2017nci}. We utilize publicly available datasets generated by The Cancer Genome Atlas (TCGA) program, which provides a comprehensive collection of clinical, molecular, and imaging data for 11,315 patients across 33 cancer types \cite{zhang2019survey}. Each patient is associated with a unique identifier, and longitudinal clinical follow-up is available, recording either the time to death or the time to last clinical observation (right censoring).

In our vertical federated learning (VFL) setup, the server holds the label file containing, for each patient, the event indicator (1 if death is observed, 0 if censored), the observed survival time, and the corresponding patient ID. The feature space is vertically partitioned across three clients based on data modality: (i) clinical data (structured tabular features), (ii) miRNA expression profiles (omics data), and (iii) DNA methylation (DNAm) profiles (omics data). This modality-wise partitioning naturally induces a three-client VFL setting, where each client owns a distinct subset of patient features while sharing aligned patient identifiers with the server.

\subsubsection{Data Preprocessing}
We built the study cohort using open-access data from The Cancer Genome Atlas (TCGA) obtained via the NCI Genomic Data Commons (GDC) portal. The preprocessing steps for each data modality are summarized below

\begin{itemize}
    \item \textit{Labels file}: Clinical metadata were downloaded in \texttt{.tsv} format. For each patient, the event indicator was defined as 1 if death was observed and 0 otherwise. The observed survival time was taken as days to death for uncensored patients and days to last follow up for censored patients. The cohort was randomly partitioned into training, validation, and test sets with an 80:10:10 split. The final labels file contains patient IDs, event indicators, survival times, and data split identifiers.

    \item \textit{Clinical data}: From the clinical metadata, we selected 9 categorical features and 1 continuous feature based on data availability. Missing values in categorical variables were imputed using the mode, and missing values in the continuous variable were imputed using the median. After preprocessing, clinical features were available for 9,729 patients. Each patient’s clinical features were stored in a separate file indexed by patient ID.


    \item \textit{miRNA data}: miRNA expression profiles were downloaded from the GDC portal by selecting the miRNA-Seq experimental strategy and using the corresponding manifest file. All 1,881 miRNA features were retained. The miRNA dataset was aligned with the labels file by keeping only patients with matching identifiers in both datasets. The processed miRNA features were stored in individual patient directories indexed by patient IDs.

    
    \item \textit{DNAm data}: DNA methylation (DNAm) profiles were obtained from methylation array experiments via the GDC portal and downloaded in \texttt{.txt} format. Due to computational constraints, the data were downloaded and processed in batches using subsets of the manifest file. During preprocessing, features were filtered by retaining genes with the highest variance, and missing values were imputed using feature-wise medians. After preprocessing, DNAm data were available for 7,500 patients. The processed DNAm features were stored on a per-patient basis in directories indexed by patient IDs.

\end{itemize}

Our data loader uses the labels file as the reference cohort and loads all patients for whom clinical data are available. If a particular modality is missing for a patient, the corresponding input is replaced with a zero tensor for that modality.



\subsection{Experimental Setup}
\subsubsection{Model Architecture}
\paragraph{Centralized model: }
The centralized model consists of four main components: 
(i) modality specific sub networks, 
(ii) a fusion layer, 
(iii) a final fully connected prediction network, and 
(iv) the loss function. The modality specific sub networks are described below. We use particularly two specific MLPs as the sub models, namely: BayesianClinicalNet and BayesianFC, which we describe in the following paragraphs: 

\textit{\textbf{BayesianClinicalNet}}: This module is a Bayesian multilayer perceptron (MLP) that acts as a feature extractor for clinical data. It employs Bayesian embedding layers for categorical features and a deterministic 1D batch normalization layer for continuous features. The normalized continuous features and categorical embeddings are concatenated into a single feature vector. Dropout with probability \(0.5\) is applied to the embedding representations for regularization. The concatenated vector is then passed through a Bayesian linear layer with 256 hidden units followed by a ReLU activation, and subsequently through a Bayesian linear projection layer that outputs a 512-dimensional latent representation.

\paragraph{\textit{BayesianFC}:}
BayesianFC is a configurable Bayesian fully connected sub-network used as a feature extractor for high-dimensional omics modalities and as the final prediction head. Each hidden block consists of a Bayesian linear layer followed by ReLU activation, batch normalization, and dropout. The width of each hidden layer is selected automatically from a predefined set of candidate sizes \(\{128, 256, 512, 1024\}\), choosing the smallest size greater than or equal to the input dimensionality. A scaling factor \(s\) is applied to increase the hidden layer width for higher-capacity networks. This scaling improves representational capacity for complex, high-dimensional modalities.

The BayesianFC architecture is instantiated as follows: for miRNA data (input size \(1881\), 3 hidden layers, scaling factor \(s=2\)); for DNAm data (input size \(3774\), 5 hidden layers, scaling factor \(s=2\)); for mRNA data (input size \(1000\), 3 hidden layers); and for the final prediction head (input size \(512\), 4 hidden layers, output size \(512\)).

For all Bayesian linear layers, we use a spike and slab prior\cite{andersen2014bayesian} on both weights and biases, where the spike distribution is \(\mathcal{N}(0, 0.001^2)\), the slab distribution is \(\mathcal{N}(0, 0.3^2)\), and the mixing probability is \(\pi=0.5\). This prior encourages sparsity while allowing large magnitude weights for learning complex patterns.

The fusion module employs an attention based mechanism. Let the modality specific embeddings be stacked into a tensor of shape \((B, M, d)\), where \(B\) is the batch size, \(M\) is the number of modalities, and \(d=512\) is the embedding dimension. Each modality embedding is passed through a Bayesian linear layer followed by a \(\tanh\) activation to obtain modality specific scores. A softmax operation is applied across modalities to obtain normalized attention weights. The fused representation is computed as a weighted sum of modality embeddings. The fused 512-dimensional vector is passed to the BayesianFC prediction head and subsequently through a final Bayesian linear risk layer, which outputs a vector of size equal to the number of discrete time intervals (30 in our experiments).

\paragraph{VFL model: } Our VFL setup learns to predict interval wise survival probabilities for cancer patients based on vertical partitioning of data where the data of the same set of patients are divided among different clients based on their modality or data type.

\textit{\textbf{Client Side:}} The clients are organized by modality type, where each modality (clinical, miRNA, mRNA, DNAm) is handled by a separate client model to extract its feature representation. The client models are the same sub models used in the centralized setting. The client holding clinical data uses the BayesianClinicalNet, while the clients holding miRNA, mRNA, and DNAm data use the BayesianFC network. Each client extracts modality specific feature vectors and sends the perturbed embeddings to the server.


\textit{\textbf{Server Side:}} The server contains the fusion layer, which uses the same attention based fusion mechanism as the centralized model. This layer fuses the modality specific feature representations received from the clients into a single compact feature vector. The server also contains the final prediction head and the risk layer to produce the interval wise survival probabilities. The labels are stored at the server to compute the training loss. The prediction head and risk layer are the same as those used in the centralized model.

\subsubsection{Model Training}

\paragraph{\textit{Centralized model}:}  
In the centralized setting, the entire network comprising the modality specific sub networks, fusion layer, prediction head, and risk layer is initialized and trained jointly in an end to end manner. For each mini batch, the modality specific sub networks extract feature representations from their respective inputs, which are fused by the attention based fusion module and subsequently passed through the prediction head and risk layer to obtain survival risk predictions. The total loss described above is computed, and gradients are obtained via backpropagation. Model parameters are updated using the Adam optimizer with a fixed learning rate. Training is performed for a fixed number of epochs with early stopping based on validation loss.


\paragraph{\textit{VFL model}:}  
In the vertical federated learning (VFL) setting, the modality specific sub networks are hosted and trained locally at their respective client sites, while the fusion module, prediction head, and risk layer are hosted on the server. During each training iteration, each client computes modality-specific embeddings using its local sub network and transmits only these intermediate embeddings to the server. The server performs feature fusion and forward propagation through the prediction head and risk layer, computes the total loss using the ground-truth labels, and backpropagates gradients with respect to the received embeddings. These gradients are then sent back to the corresponding clients, which update the parameters of their local sub-networks using backpropagation. 

Throughout training, raw features remain at the client sites and are never shared with the server. We assume a semi honest and trusted aggregator threat model for the server. 

\subsubsection{Model Optimization}
For our model we have used the AdamW optimizer\cite{loshchilov2019decoupledweightdecayregularization} which is a modification of the Adam optimizer\cite{kingma2017adammethodstochasticoptimization} by introducing decoupling of weight decay. In AdamW, a fraction of the weights are directly subtracted during weight update instead of adding $L_2$ penalty for regularization. This helps in applying weight decay uniformly across all layers without the influence of gradients, thus providing better and consistent regularization.
\subsubsection{Embedding Output Perturbation}
As discussed earlier, our defense mechanism against privacy attacks in the split neural network based VFL framework is embedding output perturbation at the client side before transmission to the server. Specifically, each client adds Gaussian noise to its clipped embedding representations prior to sharing them with the server.

Following Theorem 1, we instantiate the privacy parameters as follows. The sampling probability is set to
\[
p = \frac{\text{batch size}}{\text{training data size}}.
\]
The constant is chosen as \(c_2 = 1\), the number of iterations is given by
\[
\tau = (\text{number of epochs}) \times (\text{number of batches per epoch}),
\]
the target failure probability is set to \(\delta = 10^{-5}\), and the \(\ell_2\)-norm of the embedding outputs is clipped to \(1\), i.e.,
\[
\|f(\mathcal{B})\|_2 \leq 1.
\]
Accordingly, the noise standard deviation \(\sigma\) is selected based on the privacy budget \(\epsilon\) using the bound derived in Lemma~2.

\section{Results and Discussion}

\begin{table}[h]
    \centering
    \caption{\textsc{Comparison with baseline models on test data}}
    \label{tab:eight_cols_four_rows}
    
    \begin{tabular}{lccccccc}
        \toprule
        Model& Metrics& Clinical& miRNA& DNAm& \makecell[c]{Clinical+\\miRNA}& \makecell[c]{Clinical+\\DNAm}& \makecell[c]{Clinical+\\miRNA+\\DNAm} \\
        \midrule

        CPH& C-Index& 0.671& 0.651& 0.649&- & -& -\\
              & ctd& 0.671& 0.651& 0.644& -& -& -\\
              & IBS& 0.175& 0.195& 0.187& -& -& -\\
              & INBLL& 0.533& 0.571& 0.563& -& -& -\\
        \midrule

        DeepSurv& C-Index& 0.704& 0.649& 0.684& -&- &- \\
              & ctd& 0.704& 0.648& 0.682& -& -&- \\
              & IBS& 0.166& 0.197& 0.184&- & -& -\\
              & INBLL& 0.51& 0.583& 0.577& -&- &- \\
        \midrule

        DeepHit& C-Index& 0.672& 0.659& 0.703& -&- &- \\
              & ctd& 0.69& 0.666& 0.706&- &- &- \\
              & IBS& 0.194& 0.236& 0.243&- &- &- \\
              & INBLL& 0.573& 0.702& 0.686& -& -& -\\
        \midrule

        nnetSurvival& C-Index& 0.691& 0.613& 0.634&- &- &- \\
              & ctd& 0.701& 0.613& 0.631&- &- &- \\
              & IBS& 0.17& 0.246& 0.234& -& -&- \\
              & INBLL& 0.553& 1.403& 0.768& -&- &- \\
        \midrule

        Multisurv& C-Index& 0.696& 0.702& 0.701& 0.732& 0.737& 0.735\\
              & ctd& 0.701& 0.706& 0.703& 0.736& 0.741& 0.74\\
              & IBS& 0.163& 0.214& 0.184& 0.172& 0.188& 0.194\\
              & INBLL& 0.487& 0.808& 0.743& 0.535& 0.577& 0.648\\
        \midrule

        BayesianMultisurv& C-Index& 0.707& 0.703& 0.707& 0.734& 0.742& 0.752\\
              & ctd& 0.704& 0.705& 0.711& 0.735& 0.738& 0.755\\
              & IBS& 0.167& 0.205& 0.274& 0.176& 0.178& 0.177\\
              & INBLL& 0.472& 0.688& 0.654& 0.659& 0.604& 0.593\\
        \bottomrule
        
    \end{tabular}
        \vspace{0.2em}
    \begin{flushleft}
    \footnotesize
    ``-" indicates that the corresponding baseline does not support multimodal inputs and hence cannot be evaluated in that setting.
    \end{flushleft}
\end{table}


\begin{table}[h]
\centering
\caption{Comparison of Federated and Centralized Learning under different privacy budgets for various modalities combinations}
\label{tab:long_format_epsilon}
\renewcommand{\arraystretch}{1.2}
\begin{tabular}{c|c|cc|cc}
\hline
Modalities Combination & $\epsilon$
& \multicolumn{2}{c|}{Federated Learning (FL)}
& \multicolumn{2}{c}{Centralized Learning (CL)} \\
\cline{3-6}
& & C-Index &  & C-Index &  \\
\hline
\multirow{5}{*}{Clinical + miRNA}
& No DP ($\infty$) & \multicolumn{1}{c}{0.732}  &  & \multicolumn{1}{c}{0.743}  &  \\
& 0.5   & \multicolumn{1}{c}{0.547} &  & \multicolumn{1}{c}{0.724}  &  \\
& 1     & \multicolumn{1}{c}{0.559} &  & \multicolumn{1}{c}{0.727}  &  \\
& 1.5   & \multicolumn{1}{c}{0.565} &  & \multicolumn{1}{c}{0.731}  &  \\
& 10    & \multicolumn{1}{c}{0.712}  &  & \multicolumn{1}{c}{0.734}  &  \\
\hline
\multirow{5}{*}{Clinical + miRNA + DNAm}
& No DP ($\infty$) & \multicolumn{1}{c}{0.715} &  & \multicolumn{1}{c}{0.752}  &  \\
& 0.5   & \multicolumn{1}{c}{0.539} &  & \multicolumn{1}{c}{0.737}  &  \\
& 1     & \multicolumn{1}{c}{0.537} &  & \multicolumn{1}{c}{0.732}  &  \\
& 1.5   & \multicolumn{1}{c}{0.604} &  & \multicolumn{1}{c}{0.741}  &  \\
& 10    & \multicolumn{1}{c}{0.662} &  & \multicolumn{1}{c}{0.744}  &  \\
\hline
\end{tabular}
\end{table}

In Table \eqref{tab:eight_cols_four_rows}, we have established a detailed comparison between our centralized BayesianMultisurv model and 5 other existing baseline models including its frequen5tist counterpart, Multisurv (\cite{vale2020multisurv}). We have considered 4 standard metrics: Concordance index(C-index) \cite{antolini2005time}, time dependent concordance index(ctd), Integrated Brier Score(IBS) \cite{graf1999assessment} and Integrated Negative Binomial Log Likelihood (INBLL) for measuring the performance of the models. From the subsequent results, we can see that the BayesianMultisurv model clearly outperformed all the baseline models. With the highest C-index of 0.752, it is only followed by Multisurv, whose highest C-index is 0.737, the rest of the baseline models are significantly lagging behind our BayesianMultisurv model in terms of C-index with the highest achieved C-index among them being 0.704, which is only comparable with the lowest achieved C-index of our centralized model. It is also prominently clear that BayesianMultisurv had superior performance over the other models in terms of the other metrics as well, with Multisurv remaining its closest competitor. Also, it is to be noted that we only have unimodal results for the baseline models except Multisurv, because these models are incapable of handling multimodal data.\vspace{0.5cm}\\
In Table \eqref{tab:long_format_epsilon}, we have established a detailed comparison between our VFL model BVFLMSP and our centralized model BayesianMultisurv in terms of their respective lowest recorded validation loss values and highest recorded validation C-index values for no privacy and four levels of privacy budgets. It is to be noted that we found the optimum learning rate for our centralized Bayesianmultisurv model to be 0.005 and that of our VFL model BVFLMSP was found to be 0.001. Now, we have trained the centralized model(for 2 modalities and 3 modalities) and  the VFL model(for 2 clients) for a total of 40 epochs but due to the prolonged computational time for VFL model with 3 clients, we have restricted its training to 30 epochs. Now, from the subsequent results, it is clear that the centralized model performs better than the VFL model in all the cases which can be attributed to the inherent characteristic of VFL models, that is information fragmentation. Furthermore, this is the same reason behind the centralized model being robust to the noise induced by embedding outputs perturbation but the VFL model is sensitive to the same, our VFl model performed pretty well for \(\epsilon=10\), in sections ahead in this paper, we would show that with this decent performance \(\epsilon=10\) still guarantees differential privacy.

\subsection{Analysis of Loss and Accuracy Curves for the Centralized Model}

From Figures \eqref{fig:centralized_comparison_one} and \eqref{fig:centralized_comparison_two}, we observe that for both modality settings (Clinical + miRNA and Clinical + miRNA + DNAm), the validation loss consistently decreases and the validation concordance index (C-index) increases as training progresses. This trend is observed for all privacy budgets ($\epsilon = 0.5, 1, 1.5, 10$) as well as for the non-private setting. This shows that the centralized BayesianMultisurv model converges stably even when differential privacy noise is added to the embeddings. As expected, stronger privacy (smaller $\epsilon$) leads to slightly higher validation loss and slightly lower C-index compared to the non-private setting. However, the performance gap is small, which indicates a good privacy utility trade off.

The stable training behavior can be explained by two main reasons.

\textbf{Joint learning across modalities:}  
The centralized model jointly learns from all modalities, which provides richer and more complementary information. This helps the model compensate for the noise added for privacy preservation.

\textbf{Bayesian robustness to noise:}  
The Bayesian formulation is naturally robust to uncertainty and noise. This allows the model to absorb the embedding perturbation noise without destabilizing the training process.\vspace{0.5cm}\newline
Also, the difference in performance level for different privacy budget becomes comparatively clearer for the case with 3 modalities, because of the increase in cumulative feature space, subsequently resulting in larger volume of noise that the model has to deal with. Even then, the gap in performance level is not much significant, reaffirming our aforementioned statement about the stable behavior of training in our centralized model. 
Overall, these results confirm that adding differential privacy noise at the embedding level does not break the convergence of the centralized model and only causes a mild degradation in performance.
\begin{figure}[h]
    \centering
    \begin{subfigure}[H]{0.49\textwidth} 
        \centering
        \includegraphics[width=\linewidth]{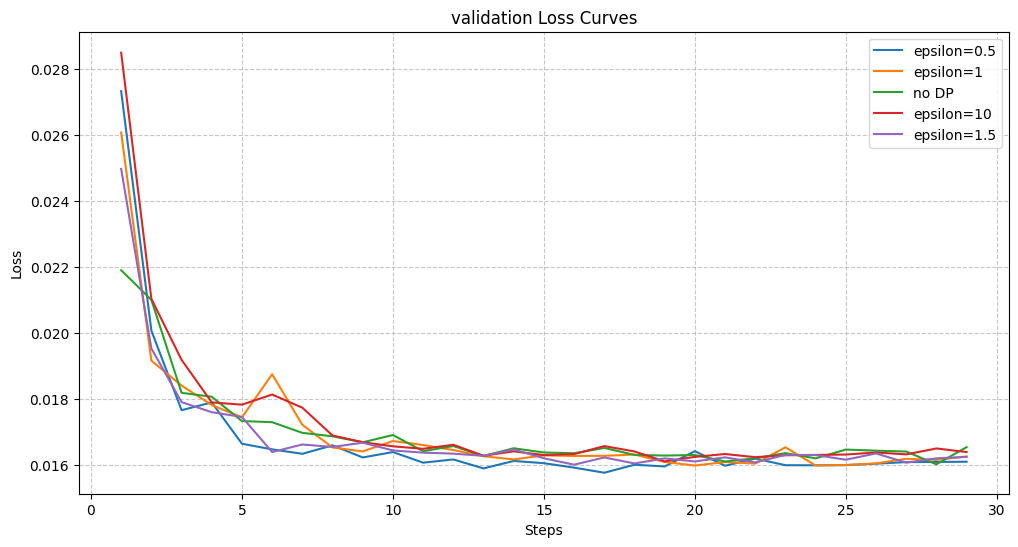}
        \caption{Loss curves} 
        \label{fig:epochs}
    \end{subfigure}
    \hfill 
    \begin{subfigure}[H]{0.49\textwidth}
        \centering
        \includegraphics[width=\linewidth]{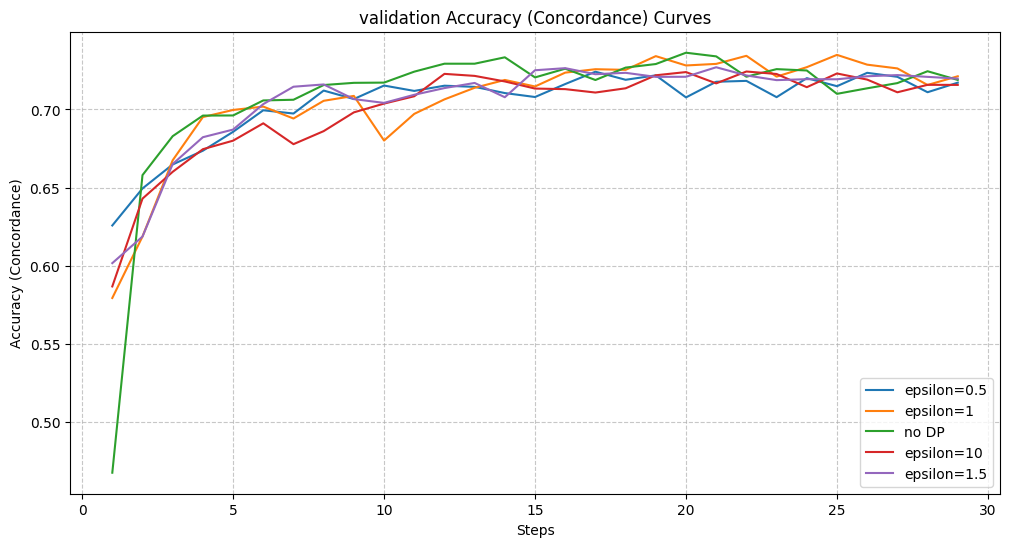}
        \caption{Accuracy curves} 
        \label{fig:time}
    \end{subfigure}
    
    \caption{Loss and accuracy curves for two modalities (Clinical + miRNA).}

    \label{fig:centralized_comparison_one}
\end{figure}
\begin{figure}[h]
    \centering
    \begin{subfigure}[H]{0.49\textwidth} 
        \centering
        \includegraphics[width=\linewidth]{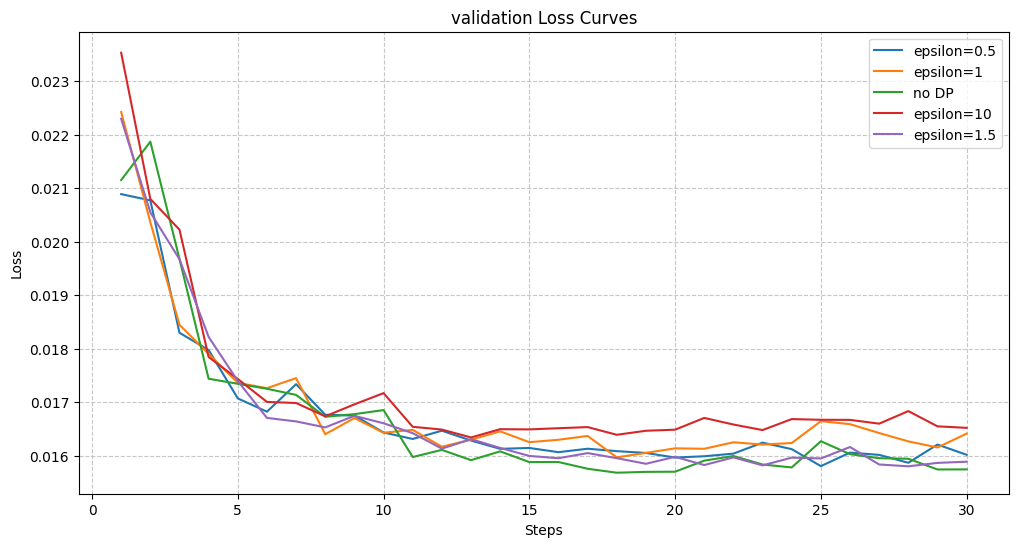}
        \caption{Loss curves} 
        \label{fig:epochs}
    \end{subfigure}
    \hfill 
    \begin{subfigure}[H]{0.49\textwidth}
        \centering
        \includegraphics[width=\linewidth]{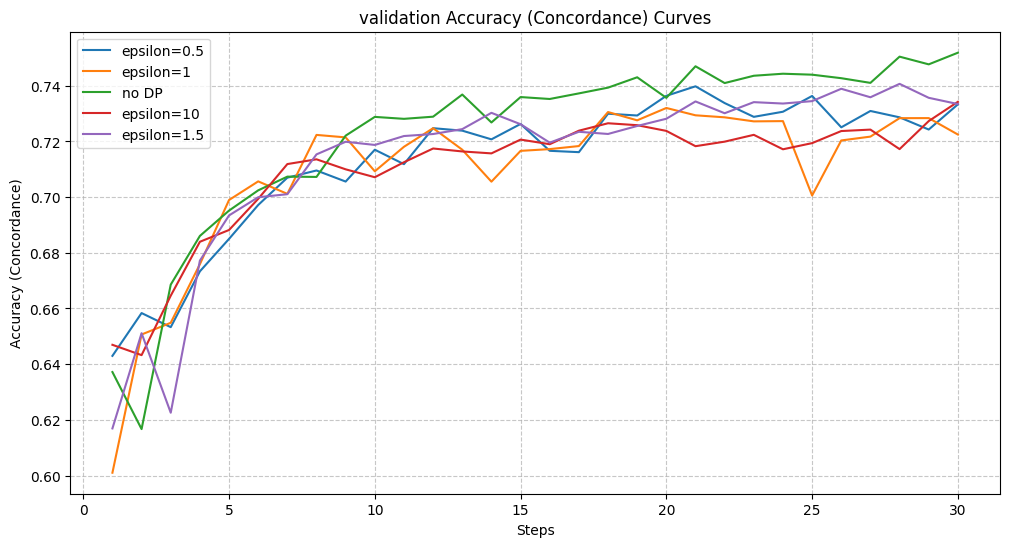}
        \caption{Accuracy curves} 
        \label{fig:time}
    \end{subfigure}
    \caption{Loss and accuracy curves for three modalities (Clinical + miRNA + DNAm).}

    \label{fig:centralized_comparison_two}
\end{figure}

\subsection{Analysis of Loss and Accuracy Curves for the Vertical Federated Learning Model}
We have considered a scenario where the hospital and pathological labs are the clients that hold different types of data for the same set of patients. The hospital holds the clinical data of the patients, one pathological lab holds the miRNA data and the other pathological lab holds the DNAm data of the patients and they run their data type specific individual sub-models. There is a server/coordinator which holds the labels file, the fusion layer, the final prediction head, risk layer and the optimizer.
In the VFL setting, each client (e.g., hospital, miRNA lab, and DNAm lab) trains its own local sub-model using only its private modality. Each client then sends a perturbed embedding to the server, where the global prediction model is trained.

Compared to the centralized case, the VFL model faces additional challenges due to data partitioning across clients and noise added to the shared embeddings for privacy preservation. Since no single client has access to complete information, the server receives partial and noisy representations, which makes optimization more difficult.

As a result, the VFL model shows slightly slower convergence and slightly lower C-index, especially under stricter privacy budgets (smaller $\epsilon$). This behavior is expected because stronger privacy introduces more noise in the embeddings, which reduces the effective information content available to the server model.

In addition, the VFL setup introduces communication noise and synchronization effects between clients and the server, which can further slow down convergence. Despite these challenges, the validation loss still decreases and the C-index still increases over training epochs, indicating stable training behavior.

These results show that the proposed embedding level differential privacy mechanism is compatible with VFL training. Although there is a small performance drop compared to the centralized setting, the model still converges reliably and achieves reasonable accuracy under strong privacy constraints.

\begin{figure}[h]
    \centering
    \begin{subfigure}[H]{0.49\textwidth} 
        \centering
        \includegraphics[width=\linewidth]{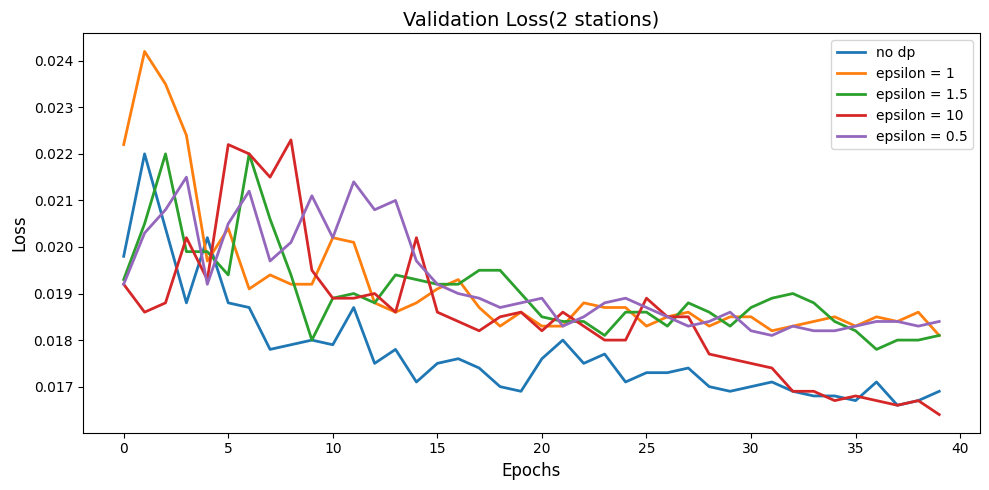}
        \caption{Loss curves} 
        \label{fig:epochs}
    \end{subfigure}
    \hfill 
    \begin{subfigure}[H]{0.49\textwidth}
        \centering
        \includegraphics[width=\linewidth]{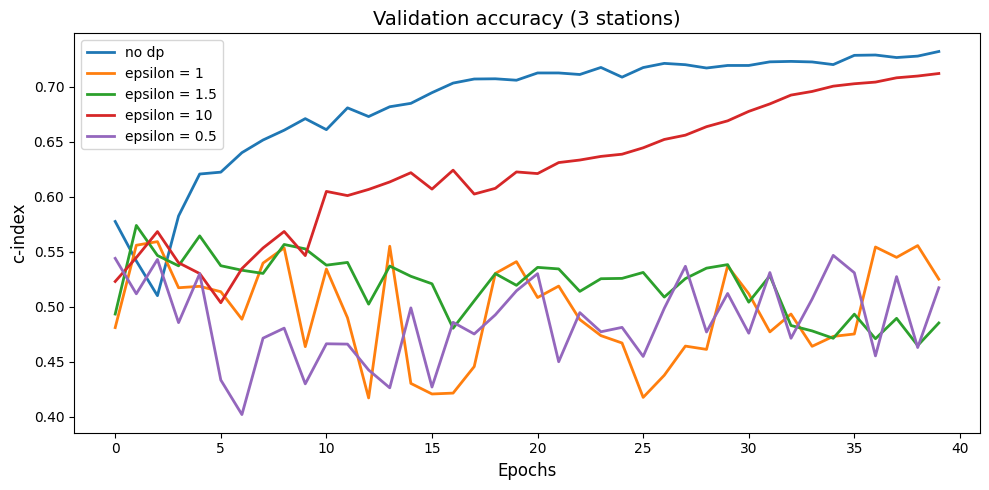}
        \caption{Accuracy curves} 
        \label{fig:time}
    \end{subfigure}
    \caption{Loss and accuracy curves for two clients (Clinical + miRNA )}
    \label{fig:non centralized_comparison_one}
\end{figure}

\begin{figure}[h]
    \centering
    \begin{subfigure}[H]{0.49\textwidth} 
        \centering
        \includegraphics[width=\linewidth]{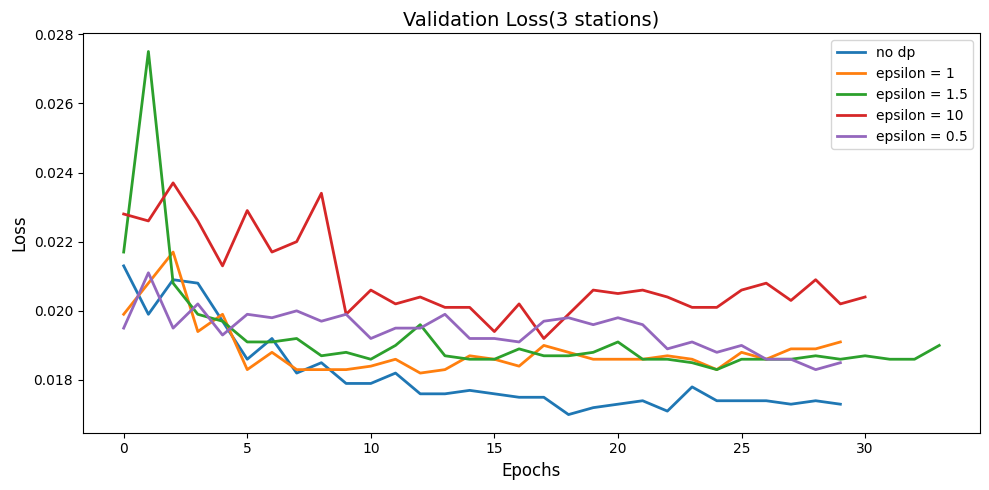}
        \caption{Loss curves} 
        \label{fig:epochs}
    \end{subfigure}
    \hfill 
    \begin{subfigure}[H]{0.49\textwidth}
        \centering
        \includegraphics[width=\linewidth]{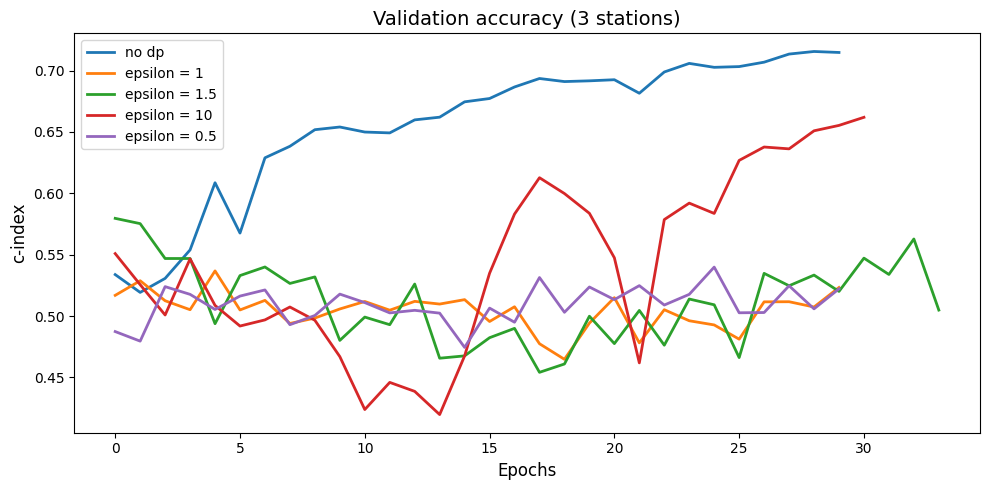}
        \caption{Accuracy curves} 
        \label{fig:time}
    \end{subfigure}
    \caption{Loss and accuracy curves for three clients (Clinical + miRNA + DNAm)}
    \label{fig:non centralized_comparison_two}
\end{figure}

In Figure \eqref{fig:non centralized_comparison_one} and \eqref{fig:non centralized_comparison_two} we can observe that for privacy budget of (\(\epsilon=10\)), the performance of differentially private BVFLMSP is almost as good as that for the non-private BVFLMSP model. Another noticeable thing is the fluctuating nature of the loss and accuracy curves for the case of the VFL model, which can be attributed to the communication noise and synchronization effects present in the VFL model as mentioned above.

\subsection{Distinguishability Between Embedding outputs of Different Patients}
This section evaluates the efficacy of the proposed defense mechanism by processing some patient samples through a client model in both protected and unprotected states. The primary objective is to determine whether the defense successfully prevents a malicious server from receiving distinct embedding outputs, thereby complicating unauthorized data recovery. For visualization purposes, we employed a perturbation level of $\epsilon = 10$. As this represents the lower bound of noise integration within our study at which BVFLMSP maintained highest performance among all privacy budgets. So, if differential privacy is preserved at this threshold, it is inherently guaranteed for all larger noise parameters that we investigated.
\begin{figure}[h]
    \centering
    \begin{subfigure}[H]{0.49\textwidth} 
        \centering
        \includegraphics[width=\linewidth]{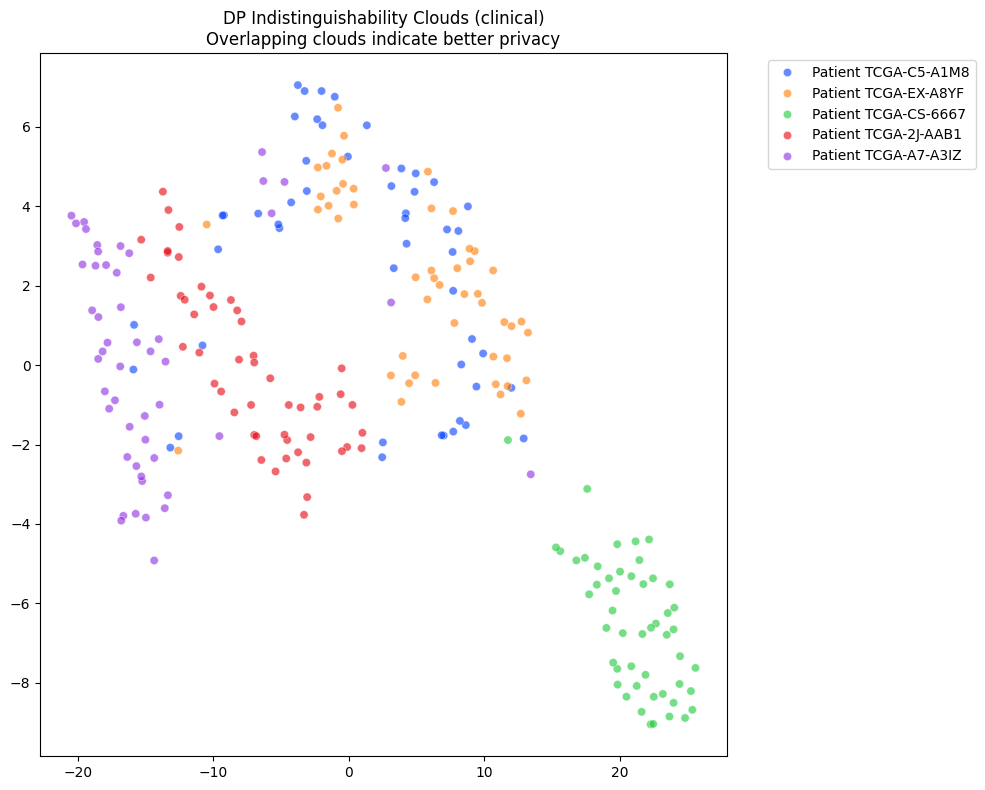}
        \caption{Without embedding output perturbation} 
        \label{fig:epochs}
    \end{subfigure}
    \hfill 
    \begin{subfigure}[H]{0.49\textwidth}
        \centering
        \includegraphics[width=\linewidth]{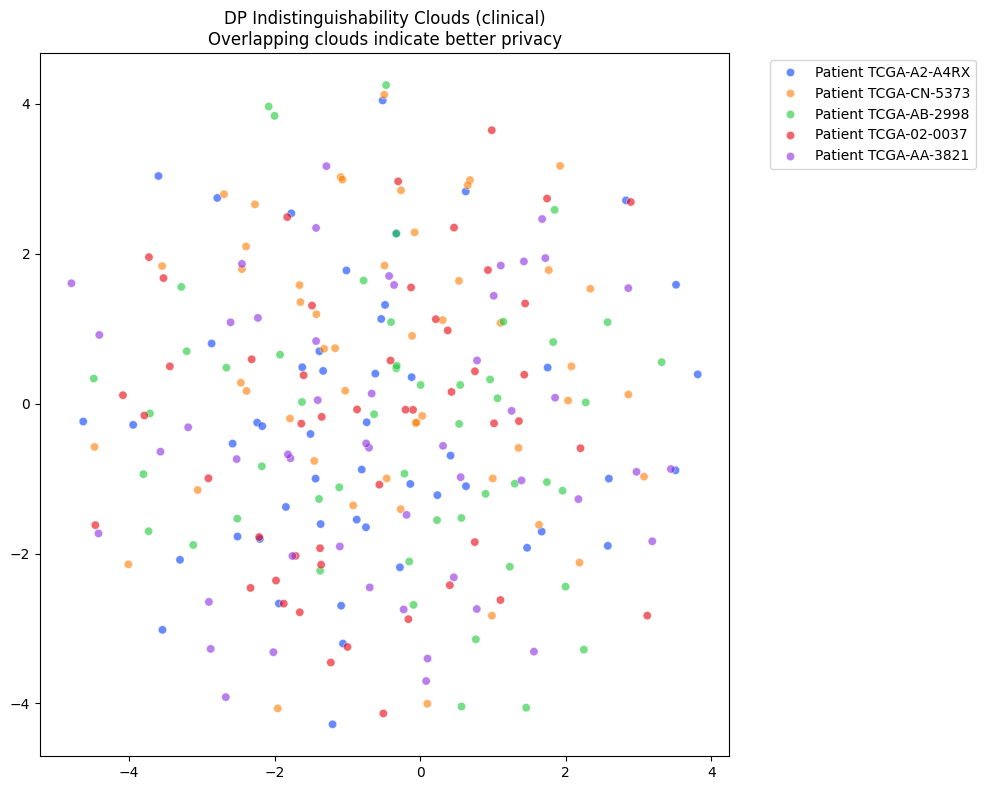}
        \caption{With embedding output perturbation} 
        \label{fig:time}
    \end{subfigure}
    \caption{Visual proof of model's differential privacy}
    \label{fig:centralized_comparison}
\end{figure}
We can clearly see that when embedding output perturbation is not applied, the embedding outputs of different patients formed distinct separate clusters of points, showing that their embedding outputs are clearly distinct. On the other hand, when embedding output perturbation is applied, the embedding outputs of different patients formed a single large gathering of points, the embedding outputs of different patients are mixed together with no clear boundary separating them. This shows how embedding output perturbation guarantees differential privacy.\newline
We draw another visualization as in Figure \eqref{fig:my_image} to visually differentiate between the inherent Bayesian noise that the client models have to deal with and the total noise (Bayesian+differential privacy) that the server has to deal with.
\begin{figure}[H]
    \centering
    \includegraphics[width=0.5\textwidth]{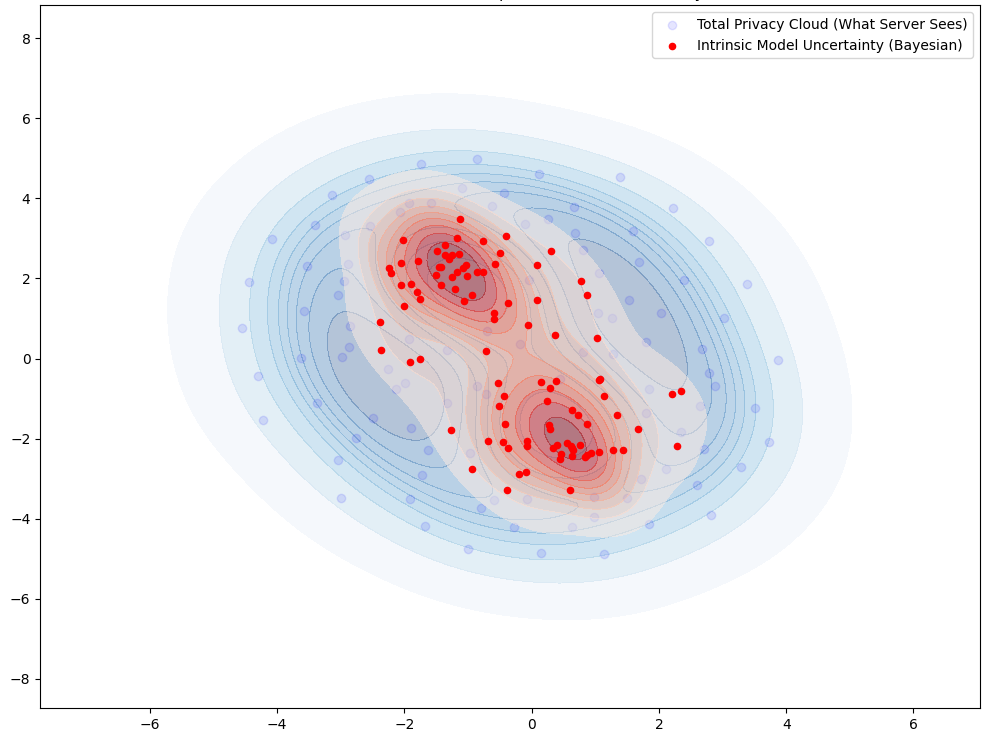}
    \caption{Composition of noise (Bayesian and Embedding output perturbation)}
    \label{fig:my_image}
\end{figure}
The noise due to inherent model uncertainty is highlighted as red and the total noise that the server has to deal with is highlighted as the blue circle containing the red region. This shows that the server has to deal with much higher amount of ambiguity further reassuring safety from privacy attack from server.

\section{Ethical Concerns}


Vertical federated learning (VFL) is an area of ongoing research and is facing continuous advancements. VFL and its increasing advancements are opening wide areas for different types of data collection for organizations to train their AI/ML models for training survival models for industrial reliability problems, financial risk managements, survival analysis for medical prognosis, etc. But these advancements poses some serious ethical questions, to what extent is it ethical to collect consumers' data. Let us consider a problem of reliability in financial sector, where a digital credit company wants to train a model on how much time consumers take to return the credit or fail to do so. For this, is it ethical to collect personal financial data of consumers including personal purchases, family spending and other details, or suppose for forming insurance policy an agency trains model based on financial data of consumers, but is it ethical and if so, to what extent. Several consumer protection legislations mandates informing consumers about the criteria impacting their insurance pricing to protect them from discrimination based pricing, but the model trainings remain an opaque space for the consumers. In this sector the strong quality of uncertainty quantification of Bayesian models can be used to produce models that are ethical enough, that is we only consider data that does not breach the consumers' privacy and also ensure reliability of the models, because the Bayesian models can account for the uncertainty introduced by the absence of certain consumer datasets. 

In our experiments on cancer survival data, clinical features contain more sensitive patient information. To follow ethical considerations, we study the effect of removing the clinical modality during training. We train two model setups: \textit{(i)} Clinical + miRNA + DNAm, and \textit{(ii)} miRNA + DNAm only. For a patient with a true survival time of 24.8 years, we pass the same input through each model 100 times using stochastic forward passes. This gives 100 predictions for each setup. Using these predictions, in the following plots we compare the predictive uncertainty of the two setups in three time intervals: early, middle, and late.

\begin{figure}[H]
    \centering
    \begin{subfigure}[H]{0.49\textwidth} 
        \centering
        \includegraphics[width=\linewidth]{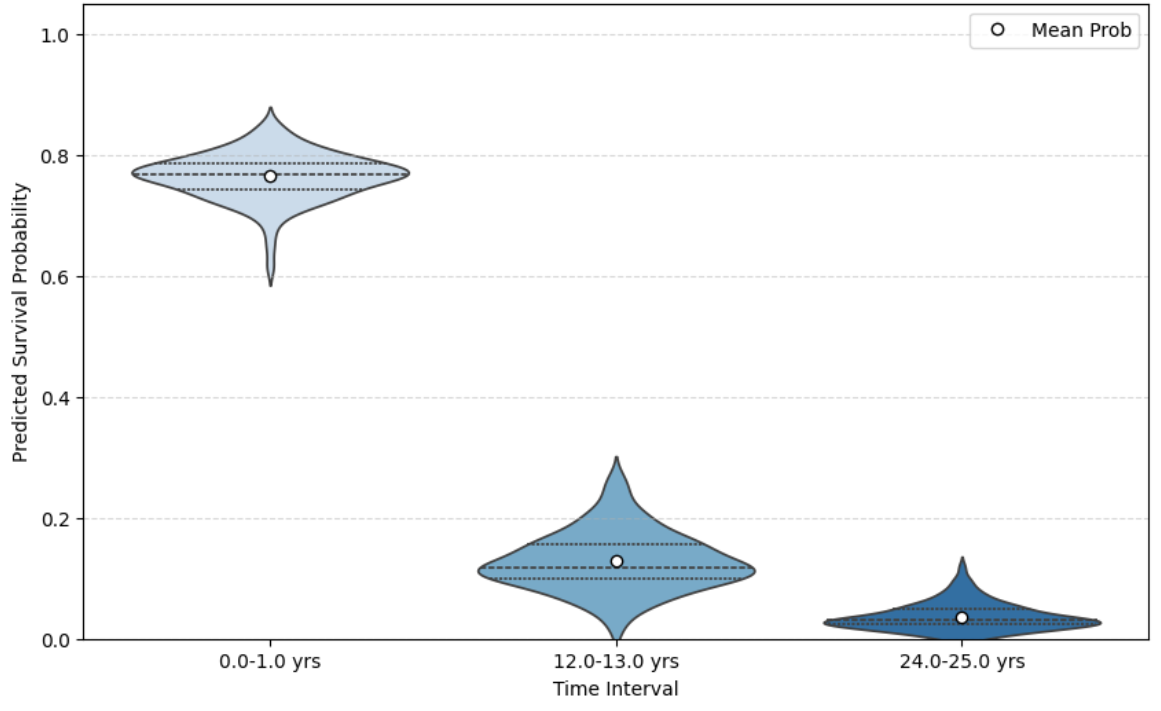}
        \caption{Clinical+DNA+miRNA} 
        \label{fig:epochs}
    \end{subfigure}
    \hfill 
    \begin{subfigure}[H]{0.49\textwidth}
        \centering
        \includegraphics[width=\linewidth]{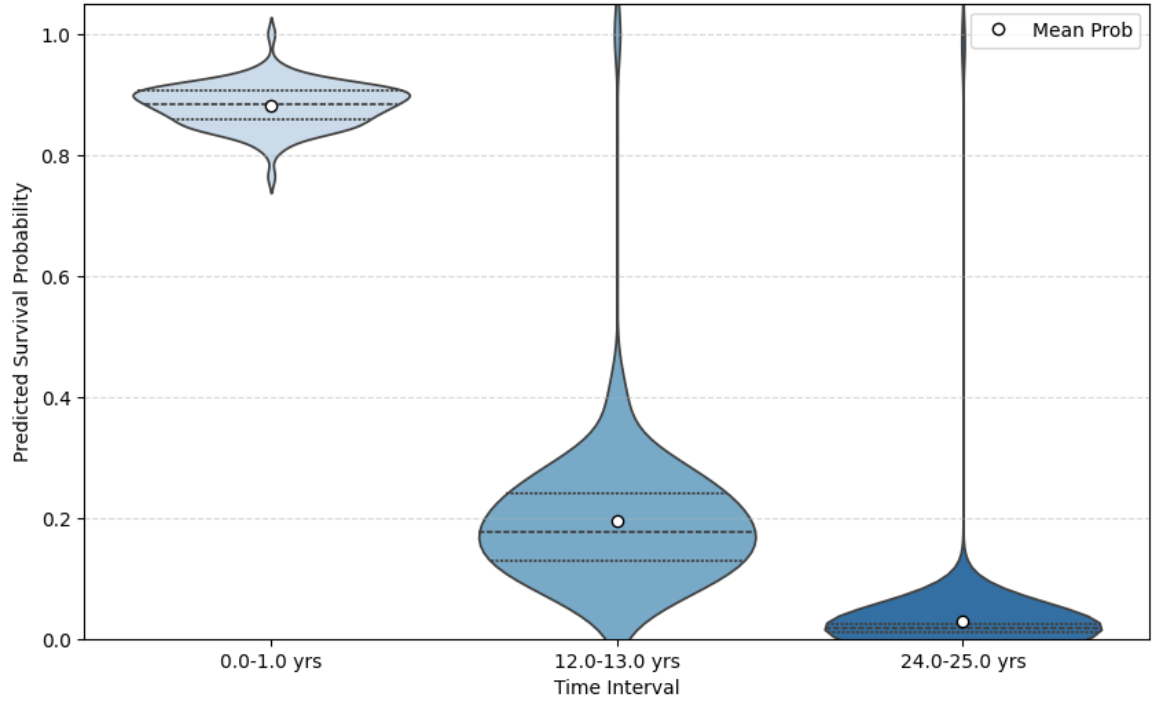}
        \caption{miRNA+DNA} 
        \label{fig:time}
    \end{subfigure}
    
    \caption{Uncertainty comparison}
    \label{fig:centralized_comparison}
\end{figure}

From the  violin plot given in Figure \ref{fig:centralized_comparison}, we observe that the predicted probabilities in the setup without clinical data show a wider spread compared to the setup that includes clinical data. This indicates a higher uncertainty in the model predictions due to the absence of complementary clinical information. We can also notice that the outputs of the model remained roughly the same. This shows that the Bayesian model's generalization and uncertainty quantification ability together have the ability to produce reliable results without violating consumer privacy.

\section{Conclusions and Future Work}

Survival analysis plays a central role in high stakes decision making across healthcare, biomedical research, and industrial reliability. The growing demand for accurate, robust, and data-efficient models has motivated the incorporation of heterogeneous and high-dimensional modalities, giving rise to multimodal survival modeling. At the same time, the need to enable collaborative learning across decentralized and siloed data sources has led to the emergence of federated learning and, more specifically, vertical federated learning (VFL). However, these collaborative paradigms introduce new privacy risks, particularly in the presence of curious or malicious servers.

In this work, we proposed a centralized Bayesian multimodal neural network for survival analysis and extended it to a SplitNN based VFL framework. To mitigate feature reconstruction and data recovery attacks, we introduced a client-side differentially private embedding perturbation mechanism. Our empirical results demonstrate that the centralized Bayesian multimodal model consistently outperforms strong baseline methods in survival prediction, achieving up to a 2.3\% relative improvement in C-index over the best multimodal baseline. Furthermore, we show that the proposed defense mechanism in the VFL setting provides formal differential privacy guarantees against feature reconstruction attacks under a semi-honest server threat model.

This work contributes one of the first end to end implementations of privacy preserving VFL for multimodal survival analysis, an area that remains underexplored in the literature. The Bayesian formulation further enhances robustness by enabling principled uncertainty quantification and resilience to missing modalities. In particular, our experiments indicate that the model can maintain competitive predictive performance even when certain modalities are absent, while explicitly capturing the uncertainty induced by missing information.

Future work will extend the proposed framework to incorporate imaging modalities (e.g., histopathology and radiology), which were not considered in this study but are central to many real world survival analysis applications. We also aim to investigate robustness under stronger adversarial threat models and study adaptive privacy utility trade offs for different privacy budgets. Our Bayesian framework can handle missing modalities and quantify uncertainty without a large drop in performance.
This opens future work on fairness, bias, and responsible use of VFL-based survival models in sensitive domains such as healthcare.



\bibliographystyle{unsrtnat}
\bibliography{references}  






\appendix
\section{Appendix}
\label{app:proof}

\begin{proof}[Proof of Theorem~\ref{thm:main}.]

We consider the optimization problem,
\begin{equation*}
\min_{\Phi}
 J(\Phi) = \frac{1}{N} \sum_{i=1}^{N}
\ell \Big(
f\big( E'^{(1)}_i, \dots, E'^{(M)}_i \big),
y_i
\Big),
\label{eq:objective}
\end{equation*}

where $\Phi = [(\theta_1)^{T}, \dots, (\theta_M)^{T}, (\theta_s)^{T}]^{T}$ denotes the
parameters of the model. $\theta_1, \dots, \theta_M$ are the local client parameters and $\theta_s$ is the central server parameters. $\Phi^*$ is optimal parameters.
The embedding of $k$-th client for the $i$-th sample is given by $E^{(k)}_i = f_k(x^{(k)}_i; \theta_k)$. Then
\begin{equation*}
    E'^{(k)}_i = E^{(k)}_i + \xi^{(k)}_i, 
\end{equation*}
where $\xi^{(k)}_i$ is the Gaussian noise, independent across clients, added to the embedding of
client $k$ for sample $i$ with mean zero and variance $(\sigma_d^{k})^2$ of the $d$-th element of the embedding. $D_k^{Embedding}$ is the dimension of the embedding of the $k$-th client. \\
The empirical risk is approximately defined as
\begin{equation*}
J(\Phi)
\cong
\frac{1}{N} \sum_{i=1}^N
F(E^{(1)}_i+\xi^{(1)}_i, \ldots,
E^{(M)}_i+\xi^{(M)}_i\big).
\end{equation*}

$F(\cdot)$ be a twice differentiable function. Using a first-order Taylor expansion of the aggregation function
$F(\cdot)$ around $E_i^{(k)}$ \cite{gai2025differentially}, we obtain
\begin{equation*}
F(E^{(1)}_i+\xi^{(1)}_i, \ldots,
E^{(M)}_i+\xi^{(M)}_i)
=
F(E^{(1)}_i, \ldots,
E^{(M)}_i)
+ \sum_{k=1}^M
\xi^{(k)}_i
\frac{\partial F}{\partial E^{(k)}_i}
+ O(\|\xi\|).
\label{eq:taylor}
\end{equation*}
We consider the higher order terms in this Taylor expansion negligible relative to the first order term because $\xi^{(k)}_i$s are small enough so that higher order terms are smaller than the linear term.

We simply assume that $\tilde{F} = F(E^{(1)}_i+\xi^{(1)}_i, \ldots, E^{(M)}_i+\xi^{(M)}_i)$. Then, we can write the partial gradient $\tilde{\nabla}_{\theta_k} J(\Phi)$ as:

\begin{align*}
\tilde{\nabla}_{\theta_k} J(\Phi)
&= \frac{1}{N} \sum_{i=1}^N
\left(
\frac{\partial \tilde{F}}{\partial E'^{(k)}_i}
\right)
\frac{\partial E'^{(k)}_i}{\partial \theta_k} \\
&= \frac{1}{N} \sum_{i=1}^N
\left(
\frac{\partial \tilde{F}}{\partial E^{(k)}_i}
\right)
\frac{\partial E^{(k)}_i}{\partial \theta_k} 
\quad \text{Since $\xi^{(k)}_i$ is noise and does not depend on $\theta_k$}\\
&= \frac{1}{N} \sum_{i=1}^N
\left(
\frac{\partial F(E^{(1)}_i, \ldots, E^{(M)}_i)}{\partial E^{(k)}_i}
+ \xi^{(k)}_i
\frac{\partial^2 {F}}{\partial (E^{(k)}_i)^2}
\right)
\frac{\partial E^{(k)}_i}{\partial \theta_k} \\
&= \nabla_{\theta_k} J(\Phi)
+ 
\frac{1}{N} \sum_{i=1}^N
\xi^{(k)}_i
\frac{\partial^2 {F}}{\partial (E^{(k)}_i)^2}
\frac{\partial (E^{(k)}_i)}{\partial \theta_k}.
\end{align*}


So, the partial gradient with respect to $\theta_k$ can be written as
\begin{equation}
\tilde{\nabla}_{\theta_k} J(\Phi)
=
\nabla_{\theta_k} J(\Phi) + a_r^k,
\label{eq:noisy_grad}
\end{equation}
where, $\nabla_{\theta_k} J(\Phi)$ is the gradient without noise. And the deviation caused by embedding noise is
\begin{equation*}
a_r^k
=
\frac{1}{N} \sum_{i=1}^N
\xi^{(k)}_i
\frac{\partial^2 {F}}{\partial (E^{(k)}_i)^2}
\frac{\partial (E^{(k)}_i)}{\partial \theta_k}.
\label{eq:or}
\end{equation*}

We define 
\begin{equation*}
U_i =
\frac{\partial^2 {F}}{\partial (E_i^{(k)})^2}
\frac{\partial (E_{i}^{(k)})}{\partial \theta_k}.
\end{equation*}

Taking the expectation of $\!\|a_r^k\|^2$ \cite{gai2025differentially}, we get
\begin{align*}
\mathbb{E}\!\left[\|a_r^k\|^2\right]
&= \frac{1}{N^2}
\mathbb{E}\!\left[
\left\|
\sum_{i=1}^N \xi^{(k)}_i U_i
\right\|^2
\right] \nonumber \\
&= \frac{1}{N^2}
\mathbb{E}\!\left[
\left(
\sum_{i=1}^N \xi^{(k)}_i U_i
\right)^{T}
\left(
\sum_{j=1}^N \xi^{(k)}_j U_j
\right)
\right] \nonumber \\
&= \frac{1}{N^2}
\mathbb{E}\!\left[
\sum_{i=1}^N \sum_{j=1}^N
(\xi^{(k)}_i U_i)^{T}
(\xi^{(k)}_j U_j)
\right] \nonumber \\
&= \frac{1}{N^2}
\sum_{i=1}^N \sum_{j=1}^N
\mathbb{E}\!\left[
(\xi^{(k)}_i U_i)^{T}
(\xi^{(k)}_j U_j)
\right] \nonumber \\
&= \frac{1}{N^2}
\sum_{i=1}^N
\mathbb{E}\!\left[
\|\xi^{(k)}_i\|^2 \|U_i\|^2
\right] \nonumber \quad \text{Since $\xi^{(k)}_i$ and $\xi^{(k)}_j$ for $i \ne j$.} \\
&= \frac{1}{N^2}
\sum_{i=1}^N
\|U_i\|^2
\mathbb{E}\!\left[
\|\xi^{(k)}_i\|^2
\right] \nonumber \\
&= \frac{1}{N^2}
\sum_{i=1}^N
\left(
\sum_{d=1}^{d^{Embedding}} (\sigma_d^k)^2
\right)
\|U_i\|^2 .
\end{align*}

Since $\xi^{(k)}_i$s are independent zero-mean Gaussian noises with the same variance as $\sigma^2C^2$, the
expected squared norm of $o_r^k$ is written as
\begin{align*}
\mathbb{E}\!\left[\|a_r^k\|^2\right]
=
\frac{1}{N^2}
\sum_{i=1}^N
\left(
\sum_{d=1}^{d^{Embedding}} (\sigma_d^k)^2
\right)
\|U_i\|^2 \\
= \frac{1}{N^2}
d^{Embedding}*\sigma^2C^2
\sum_{i=1}^N \|U_i\|^2.
\label{eq:or_bound}
\end{align*}

Let $\Phi^{(e+1)} = \left[ (\theta_1^{(e+1)})^{T}, \ldots, (\theta_M^{(e+1)})^{T} \right]^{T}$, 
where $\Phi^{(e+1)}$ denotes the local client parameters for the $(e+1)$-th epoch. Then, the model parameters are updated as
\begin{equation*}
\Phi^{(e+1)} = \Phi^{(e)} - \eta \tilde{\nabla} J(\Phi^{(e)}).
\end{equation*}

Here
\begin{equation}
    \tilde{\nabla} J(\Phi^{(e)}) =
    \underbrace{
    \left[
    (\nabla_{\theta_1^{(e)}} J(\Phi^{(e)}))^{T},
    \ldots,
    (\nabla_{\theta_M^{(e)}} J(\Phi^{(e)}))^{T}
    \right]^{T}
    }_{\nabla J(\Phi^{(e)})}
    +
    \underbrace{
    \left[
    (a_1^{r})^{T},
    \ldots,
    (a_M^{r})^{T}
    \right]^{T}
    }_{a^{r}}.
    \label{eq:total_loss}
\end{equation}


From Assumption 5, we have
\begin{align*}
J(\Phi^{(e+1)})
&\le
J(\Phi^{(e)})
+ \nabla J(\Phi^{(e)})^{T} (\Phi^{(e+1)} - \Phi^{(e)})
+ \frac{\beta}{2} \|\Phi^{(e+1)} - \Phi^{(e)}\|^2 \nonumber \\
&=
J(\Phi^{(e)})
- \eta \nabla J(\Phi^{(e)})^{T} \tilde{\nabla} J(\Phi^{(e)})
+ \frac{\beta \eta^2}{2}
\|\tilde{\nabla} J(\Phi^{(e)})\|^2 
\\
&= 
J(\Phi^{(e)})
- \eta \nabla J(\Phi^{(e)})^{T} ({\nabla} J(\Phi^{(e)}) + a_r)
+ \frac{\beta \eta^2}{2}
\|{\nabla} J(\Phi^{(e)}) + a_r\|^2
\\
&=
J(\Phi^{(e)})
- \eta \nabla J(\Phi^{(e)})^{T} {\nabla} J(\Phi^{(e)}) - \eta \nabla J(\Phi^{(e)})^{T}a_r \\
& \quad + \frac{\beta \eta^2}{2}
(\|{\nabla} J(\Phi^{(e)})\|^2 + \|a_r\|^2 + 2{\nabla} J(\Phi^{(e)})^{T}a_r).
\end{align*}

Setting the learning rate $\eta = 1/\beta$, we get
\begin{align*}
J(\Phi^{(e+1)})
&\le
J(\Phi^{(e)})
- \frac{1}{\beta} \nabla J(\Phi^{(e)})^{T} {\nabla} J(\Phi^{(e)}) - \frac{1}{\beta} \nabla J(\Phi^{(e)})^{T}a_r \\
& \quad + \frac{1}{2\beta}
(\|{\nabla} J(\Phi^{(e)})\|^2 + \|a_r\|^2 + 2{\nabla} J(\Phi^{(e)})^{T}a_r)
\\
&=
J(\Phi^{(e)})
- \frac{1}{\beta} \|\nabla J(\Phi^{(e)})\|^2
- \frac{1}{\beta} \nabla J(\Phi^{(e)})^{T}a_r
+ \frac{1}{2\beta}
\|{\nabla} J(\Phi^{(e)})\|^2 + \frac{1}{2\beta} \|a_r\|^2 + \frac{1}{\beta} {\nabla} J(\Phi^{(e)})^{T}a_r
\\
&=
J(\Phi^{(e)})
- \frac{1}{2\beta} \|\nabla J(\Phi^{(e)})\|^2
+ \frac{1}{2\beta} \|a_r\|^2
\\
&=
J(\Phi^{(e)})
+ \frac{1}{2\beta} (\|a_r\|^2 - \|\nabla J(\Phi^{(e)})\|^2).
\end{align*}

We can write
\begin{equation}
J(\Phi^{(e+1)})
\le
J(\Phi^{(e)})
+ \frac{1}{2\beta}
\big(
\|a_r\|^2 - \|\nabla J(\Phi^{(e)})\|^2
\big).
\label{eq:descent}
\end{equation}

From Assumption~4 and similar calculation of equation \eqref{eq:descent}, we have
\begin{align*}
J(\Phi^*)
&\ge
J(\Phi^{(e)})
- \frac{1}{2\alpha}
\left\|
\nabla J(\Phi^{(e)})
\right\|^2
+ \frac{1}{2\alpha}
\left\|
a_r
\right\|^2 \nonumber \\
&\ge
J(\Phi^{(e)})
- \frac{1}{2\alpha}
\left\|
\nabla J(\Phi^{(e)})
\right\|^2 .
\end{align*}

Then we can write
\begin{equation}
\|\nabla J(\Phi^{(e)})\|^2
\ge
2\alpha \big(J(\Phi^{(e)}) - J(\Phi^*)\big).
\label{eq:grad_lower}
\end{equation}

From equation \eqref{eq:descent} and equation \eqref{eq:grad_lower}, we get

\begin{align*}
J(\Phi^{(e+1)})
&\le
J(\Phi^{(e)})
+ \frac{1}{2\beta}
\big(
\|a_r\|^2 - \|\nabla J(\Phi^{(e)})\|^2
\big)
\\
&\le
J(\Phi^{(e)})
+ \frac{1}{2\beta}
\big\{
\|a_r\|^2 - 2\alpha \big(J(\Phi^{(e)}) - J(\Phi^*)\big)
\big\}
\\
&=
J(\Phi^{(e)})
+ \frac{1}{2\beta}
\big
\|a_r\|^2 - \frac{\alpha}{\beta} \big(J(\Phi^{(e)}) - J(\Phi^*)\big)
\\
&=
J(\Phi^{(e)}) - \frac{\alpha}{\beta} \big(J(\Phi^{(e)}) - J(\Phi^*)\big) + \frac{1}{2\beta}
\big
\|a_r\|^2.
\end{align*}

Subtracting $J(\Phi^*)$ on both side, we can write
\begin{align}
J(\Phi^{(e+1)}) - J(\Phi^*)
&\le
J(\Phi^{(e)}) - J(\Phi^*)
- \frac{\alpha}{\beta}
\big( J(\Phi^{(e)}) - J(\Phi^*) \big)
+ \frac{1}{2\beta}
\left\| a_r \right\|^2 \nonumber \\
&=
\left(1 - \frac{\alpha}{\beta}\right)
\big( J(\Phi^{(e)}) - J(\Phi^*) \big)
+ \frac{1}{2\beta}
\left\| a_r \right\|^2.
\label{eq:rec_equ}
\end{align}

Recursively applying \eqref{eq:rec_equ}, we obtain
\begin{align}
J(\Phi^{L}) - J(\Phi^*)
\le\;&
\left(1 - \frac{\alpha}{\beta}\right)^{L}
\big( J(\Phi^{(0)}) - J(\Phi^*) \big) +
\frac{1}{2\beta}
\sum_{e=0}^{L-1}
\left(1 - \frac{\alpha}{\beta}\right)^{L-e-1}
\left\| a_r \right\|^2.
\label{eq:final_equ}
\end{align}

Taking the expectation on both sides of \eqref{eq:final_equ}, we have
\begin{align*}
\mathbb{E}\!\left[ J(\Phi^{L}) - J(\Phi^*) \right]
\le\;&
\left(1 - \frac{\alpha}{\beta}\right)^{L}
\mathbb{E}\!\left[ J(\Phi^{(0)}) - J(\Phi^*) \right] \nonumber 
+
\frac{1}{2\beta}
\sum_{e=0}^{L-1}
\left(1 - \frac{\alpha}{\beta}\right)^{L-e-1}
\mathbb{E}\!\left[ \left\| a_r\right\|^2 \right] \nonumber \\
\le\;&
\left(1 - \frac{\alpha}{\beta}\right)^{L}
\mathbb{E}\!\left[ J(\Phi^{(0)}) - J(\Phi^*) \right] \nonumber \\
&+
\frac{1}{2N^2\beta}
\sum_{e=0}^{L-1}
\left(1 - \frac{\alpha}{\beta}\right)^{L-e-1}
\sum_{i=1}^N \left(\sum_{k=1}^{M} d^{Embedding}*\sigma^2C^2\right)
\|U_i(e)\|^2.
\end{align*}

Here we denote $U_i$ by $U_i(e)$ because in each epoch $\theta_k$ becomes $\theta_k^{(e)}$. \\
In our model, we take the dimension of the final embedding of each of the client as 512 ($d^{Embedding} = 512$) and C is taken as 1. The value of $\sigma$ will be determined from Equation \eqref{eq:sigma}.
So, the above expression is written as

\begin{align*}
\mathbb{E}\!\left[ J(\Phi^{L}) - J(\Phi^*) \right]
\le\;&
\left(1 - \frac{\alpha}{\beta}\right)^{L}
\mathbb{E}\!\left[ J(\Phi^{(0)}) - J(\Phi^*) \right] \nonumber \\
&+
\frac{512*\sigma^2*M}{2N^2\beta}
\sum_{e=0}^{L-1}
\left(1 - \frac{\alpha}{\beta}\right)^{L-e-1}
\sum_{i=1}^N \|U_i(e)\|^2 \\
\le\;&
\left(1 - \frac{\alpha}{\beta}\right)^{L}
\mathbb{E}\!\left[ J(\Phi^{(0)}) - J(\Phi^*) \right] \nonumber \\
&+
\frac{256*\sigma^2*M}{N^2\beta}
\sum_{e=0}^{L-1}
\left(1 - \frac{\alpha}{\beta}\right)^{L-e-1}
\sum_{i=1}^N \|U_i(e)\|^2.
\end{align*}


From Assumption 5, we obtain, $|| \frac{\partial^2 {F}}{\partial (E_{i}^{(k)})^2}|| \le \beta$. 
Also, $\theta_k^{(e)}$s are shrunk by $L_2$ regularization \cite{goodfellow2016deep}, $||\frac{\partial (E_{i}^{(k)})}{\partial \theta_k^{(e)}}|| \le L_E$.

Then
\[|| U_i(e) || \le \beta * L_E.\]

So, the final expression is written as
\begin{align*}
\mathbb{E}\!\left[ J(\Phi^{L}) - J(\Phi^*) \right]
\le\;&
\left(1 - \frac{\alpha}{\beta}\right)^{L}
\mathbb{E}\!\left[ J(\Phi^{(0)}) - J(\Phi^*) \right] \nonumber \\
&+
\frac{256*\sigma^2*M}{N^2\beta}
\sum_{e=0}^{L-1}
\left(1 - \frac{\alpha}{\beta}\right)^{L-e-1}
\sum_{i=1}^N \|U_i(e)\|^2 \\
\le\;&
\left(1 - \frac{\alpha}{\beta}\right)^{L}
\mathbb{E}\!\left[ J(\Phi^{(0)}) - J(\Phi^*) \right] \nonumber \\
&+
\frac{256*\sigma^2*M}{N^2\beta}
\sum_{e=0}^{L-1}
\left(1 - \frac{\alpha}{\beta}\right)^{L-e-1}
* N * \beta * L_E \\
\le\;&
\left(1 - \frac{\alpha}{\beta}\right)^{L}
\mathbb{E}\!\left[ J(\Phi^{(0)}) - J(\Phi^*) \right] \nonumber \\
&+
\frac{256*\sigma^2*M*L_E*\beta}{N*\alpha}
*\left[1 - \left(1-\frac{\alpha}{\beta}\right)^L\right].
\end{align*}



Based on Assumptions 4 and 5, it implies that the term $(1- \frac{\alpha}{\beta})^L$ decays exponentially with respect to $L$. Then $\mathbb{E}\!\left[ J(\Phi^{L}) - J(\Phi^*) \right]$ approaches a bounded value with epochs.


\end{proof}

\end{document}